\documentclass[letterpaper,11pt,oneside]{article}
\usepackage[margin=2.95cm]{geometry}
\usepackage{cite}

\usepackage{microtype}
\usepackage{graphicx}
\usepackage{subfigure}
\usepackage{booktabs} 
\usepackage{hyperref}
\usepackage{tikz}
\usepackage{enumitem}

\usepackage{amsmath, amssymb, amsthm}

\usepackage{float}

\usepackage[utf8]{inputenc}

\usepackage[toc,page]{appendix}
\usepackage{titling}

\theoremstyle{plain}
\newtheorem{theorem}{Theorem}[section]
\newtheorem{lemma}[theorem]{Lemma}
\theoremstyle{definition}

\title{A Theoretical Comparison of Graph Neural Network Extensions \vspace{80pt}}

\author{\Large Pál András Papp \footnotesize \vspace{8pt} \\ ETH Zürich \vspace{2pt} \\ apapp@ethz.ch
\and \and \Large Roger Wattenhofer \footnotesize \vspace{8pt} \\ ETH Zürich \vspace{2pt} \\ wattenhofer@ethz.ch}
\date{\vspace{35pt}}

\begin{document}

\begin{titlingpage}
\maketitle
\vspace{90pt}
\begin{abstract}
	{We study and compare different Graph Neural Network extensions that increase the expressive power of GNNs beyond the Weisfeiler-Leman test. We focus on (i) GNNs based on higher order WL methods, (ii) GNNs that preprocess small substructures in the graph, (iii) GNNs that preprocess the graph up to a small radius, and (iv) GNNs that slightly perturb the graph to compute an embedding. We begin by presenting a simple improvement for this last extension that strictly increases the expressive power of this GNN variant. Then, as our main result, we compare the expressiveness of these extensions to each other through a series of example constructions that can be distinguished by one of the extensions, but not by another one. We also show negative examples that are particularly challenging for each of the extensions, and we prove several claims about the ability of these extensions to count cliques and cycles in the graph.} 
\end{abstract}
	
\vspace{5pt}
	
\end{titlingpage}

\setcounter{page}{2}

\pagenumbering{arabic}

\section{Introduction}

Due to the prominence of graph-structured data in numerous applications, Graph Neural Networks (GNNs) have been one of the main success stories in machine learning in the past few years. GNNs have produced state-of-the-art results in a wide range of areas, including quantum chemistry, molecule recognition, recommendation systems or social networks \cite{gilmer2017neural, fout2017protein, ying2018graph, sanchez2020learning}.

From a theoretical perspective, one of the most fundamental questions about GNNs is their expressive power, i.e. what are the things that GNNs can and cannot compute. In this sense, the most important limitation of standard GNNs is that their expressiveness is upper bounded by the so-called Weisfeiler-Leman (or $1$-WL) test. This implies that GNNs can sometimes not even distinguish very simple graphs. Hence there were various suggestions to develop GNN extensions with expressive power beyond $1$-WL, e.g. by augmenting the GNN with subgraph counts, or by letting the GNN also observe perturbed variants of the graph.

However, many open questions remain regarding the expressive power of these GNN extensions, and in particular, about how the different extensions relate to each other in terms of expressiveness. Whenever a new GNN variant is introduced, the corresponding theoretical analysis usually shows it to be more powerful than $1$-WL, and sometimes also compares it to the classical $k$-WL hierarchy. However, it is unclear whether the WL hierarchy is the correct tool to measure the expressiveness of these extensions, since in contrast to the locality of GNNs, $k$-WL is a concept of global comparison over two graphs. Furthermore, the hierarchy is too coarse to be useful in practice: any meaningful GNN extension is already beyond $1$-WL, whereas $3$-WL is already so powerful that it essentially recognizes all graphs apart from some highly artificial counterexamples. Indeed, many GNN extensions are only compared directly to $2$-WL \cite{bouritsas2020improving, simplicial, reconstruction}.

This raises a natural question: Can we find a more meaningful way to measure the expressiveness of GNN extensions? In fact, the ideas behind the extensions themselves already point out a straightforward way to do this. That is, if we focus on approaches that preserve the locality and permutation-equivariance of standard GNNs, we find that GNN extensions in the literature are essentially based around three main ideas, and each of these ideas provides a natural way to define an alternative hierarchy of expressiveness for GNNs.

Our main goal in the paper is to study and compare these alternative hierarchies, and hence indirectly to compare the expressive power of the different GNN extensions themselves. In particular, we consider: (i) GNNs where subgraphs up to size $k$ are preprocessed initially ($S_k$), (ii) GNNs where the $k$-hop induced neighborhood of each node is preprocessed initially ($N_k$), and (iii) GNNs where $k$ nodes are removed or marked for symmetry breaking ($M_k$). Each of these is a prominent approach to increase the expressiveness of GNNs, and has been studied before in several theoretical or empirical works. We compare the expressive power of these GNN extensions to each other and to the classical WL hierarchy.

Our main contributions are as follows:
\begin{itemize}[topsep=5pt,itemsep=0pt,partopsep=5pt,parsep=5pt]
    \item We identify three main approaches used in the literature to increase the expressive power of GNNs, and we define alternative hierarchies of expressiveness based on the largest expressive power that can be achieved with each approach.
    \item For the symmetry breaking approach, we introduce and study a new GNN extension (named GNNs with markings), and we show that it is strictly more expressive than previous extensions of this type.
    \item As our main result, we compare the expressive power of different GNN extensions by showing specific graph constructions that can be distinguished by one extension, but not by another one. We illustrate a summary of our findings in Figure \ref{fig:sum}. Note that in many cases, there is no strict ordering of expressiveness among the GNN variants: for a specific pair of extensions, we find that both of them can distinguish some graphs that the other one cannot. We also point out some cases where one extension is strictly superior to another one in terms of expressiveness.
    \item As an alternative measure of expressiveness, we prove several (positive and negative) results on the ability of the extensions to count simple substructures in a graph, such as cliques or cycles.
\end{itemize}

\begin{figure}
\centering
\resizebox{0.48\textwidth}{!}{\begin{tikzpicture}

	\draw(100pt,100pt) rectangle (150pt,150pt);
	
	\node[anchor=center] at (125pt,125pt) {\normalsize $1$-WL};
	
	\draw(105pt,150pt) rectangle (145pt,165pt);
	\draw(105pt,165pt) rectangle (145pt,180pt);
	\draw(105pt,180pt) rectangle (145pt,195pt);
	\draw(105pt,195pt) rectangle (145pt,210pt);
	\draw(105pt,210pt) rectangle (145pt,225pt);
	\draw[ultra thick](105pt,210pt) -- (145pt,210pt);
	
	\node[anchor=center] at (125pt,157.5pt) {\normalsize $S_3$};
	\node[anchor=center] at (125pt,172.5pt) {\normalsize $S_4$};
	\node[anchor=center] at (125pt,187.5pt) {\normalsize $S_5$};
	\node[anchor=center] at (125pt,202.5pt) {\normalsize $S_6$};
	\node[anchor=center] at (125pt,217.5pt) {\normalsize $S_k$};
	
	\draw(105pt,100pt) rectangle (145pt,85pt);
	\draw[ultra thick](105pt,85pt) -- (145pt,85pt);
	\draw(105pt,85pt) rectangle (145pt,70pt);
	
	\node[anchor=center] at (125pt,92.5pt) {\normalsize $2$-WL};
	\node[anchor=center] at (125pt,77.5pt) {\normalsize $k$-WL};
	
	\draw(85pt,105pt) rectangle (100pt,145pt);
	\draw(70pt,105pt) rectangle (85pt,145pt);
	\draw(55pt,105pt) rectangle (70pt,145pt);
	\draw[ultra thick](55pt,105pt) -- (55pt,145pt);
	\draw(40pt,105pt) rectangle (55pt,145pt);
	
	\node[anchor=center] at (92.5pt,125pt) {\normalsize $N_1$};
	\node[anchor=center] at (77.4pt,125pt) {\normalsize $N_2$};
	\node[anchor=center] at (62.4pt,125pt) {\normalsize $N_3$};
	\node[anchor=center] at (47.3pt,125pt) {\normalsize $N_k$};
	
	\draw(150pt,105pt) rectangle (165pt,145pt);
	\draw(165pt,105pt) rectangle (180pt,145pt);
	\draw[ultra thick](180pt,105pt) -- (180pt,145pt);
	\draw(180pt,105pt) rectangle (195pt,145pt);
	
	\node[anchor=center] at (157.4pt,125pt) {\normalsize $M_1$};
	\node[anchor=center] at (172.3pt,125pt) {\normalsize $M_2$};
	\node[anchor=center] at (187.4pt,125pt) {\normalsize $M_k$};
	
	\draw[very thick, gray, arrows=-latex] (92.5pt,105pt) -- (92.5pt,60pt) -- (184pt,60pt)  -- (184pt,105pt);
	\draw[very thick, gray, arrows=-latex] (92.5pt,77.5pt) -- (105pt,77.5pt);
	
	\draw[white, fill=white](182pt,88pt) rectangle (186pt,92pt);
	\draw[very thick, gray, arrows=-latex] (145pt,90pt) -- (205pt,90pt) -- (205pt,221pt)  -- (145pt,221pt);
	\draw[very thick, gray, arrows=-latex] (191pt,90pt) -- (191pt,105pt);
	
	\draw[white, fill=white](91pt,94pt) rectangle (93pt,98pt);
	\draw[very thick, gray, arrows=-latex] (105pt,96pt) -- (47.5pt,96pt) -- (47.5pt,105pt);
	
	\draw[very thick, gray, arrows=-latex] (145pt,202.5pt) -- (172.5pt,202.5pt) -- (172.5pt,145pt);
	\draw[very thick, gray, arrows=-latex] (145pt,172.5pt) -- (155pt,172.5pt) -- (155pt,145pt);
	
	\draw[very thick, gray, arrows=-latex] (105pt,172.5pt) -- (89pt,172.5pt) -- (89pt,145pt);
	\draw[very thick, gray, arrows=-latex] (105pt,187.5pt) -- (77.5pt,187.5pt) -- (77.5pt,145pt);
	\draw[very thick, gray, arrows=-latex] (105pt,202.5pt) -- (62.5pt,202.5pt) -- (62.5pt,145pt);
	\draw[white, fill=white](94pt,170pt) rectangle (98pt,205pt);
	\draw[very thick, gray, arrows=-latex] (96pt,145pt) -- (96pt,217.5pt)  -- (105pt,217.5pt);
	
	\draw[white, fill=white](159pt,170pt) rectangle (163pt,225pt);
	\draw[very thick, gray, arrows=-latex] (161pt,145pt) -- (161pt,235pt) -- (47.5pt,235pt)  -- (47.5pt,145pt);
	\draw[very thick, gray, arrows=-latex] (161pt,214pt) -- (145pt,214pt);
	
	\draw[white, fill=white](80pt,170.5pt) rectangle (40pt,174.5pt);
	\draw[very thick, gray, arrows=-latex] (89pt,172.5pt) -- (30pt,172.5pt) -- (30pt,89pt)  -- (105pt,89pt);
	\draw[white, fill=white](90.5pt,90pt) rectangle (94.5pt,88pt);
	\draw[very thick, gray] (92.5pt,90.5pt)  -- (92.5pt,87.5pt);
	
	\draw[very thick, gray, arrows=-latex] (172.5pt,105pt) -- (172.5pt,95pt) -- (145pt,95pt);
	
\end{tikzpicture}}
\caption{Illustration of our main result: the expressiveness relations between different extensions. We draw an arrow from extension $A_i$ to extension $B_j$ when $A_i$ can distinguish a pair of graphs that $B_j$ cannot (denoted by $A_i \succ B_j$). If $A_i$ is more expressive than $B_j$ for multiple parameters $j$, we only draw an arrow to the highest such $j$, with a specific box (labeled `$B_k$') denoting that $A_i \succ B_j$ for any choice of $j$.}
\label{fig:sum}
\end{figure}
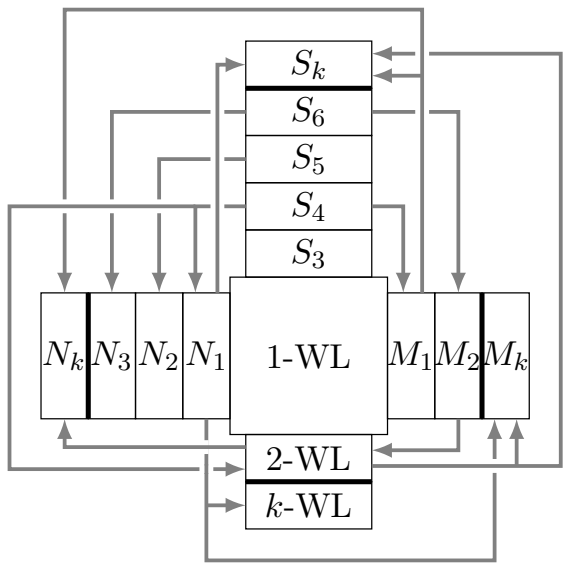

\section{Related Work}

Graph Neural Networks have been extensively studied throughout the last decade from various perspectives \cite{scarselli2008graph, wu2020comprehensive}. One of the most fundamental questions about GNNs is their theoretical expressiveness (and limitations), i.e. what GNNs can (or cannot) compute; this has also received a lot of attention in the last few years \cite{limits, limits2, limits3}.

In terms of expressive power, the most well-known limitation of message passing GNNs is that they are at most as powerful as the $1$-WL test (see Section \ref{sec:GNN}). Due to this, there were numerous studies in recent years on extending the base GNN model in different ways in order to increase its expressiveness beyond this upper bound. There are several natural ideas that appear repeatedly among these works. In this paper, we focus on the ideas that keep both the locality and permutation-equivariance properties of standard GNNs, e.g. by directly processing some substructures in the graph, or introducing small structural modifications to the graph. Since these methods are the main focus of our paper, we discuss them in more detail in Section \ref{sec:models}.

Another line of work considers higher order methods; unlike standard GNNs, these process the graph in a more global fashion, and require significantly more time/memory. We also discuss the corresponding hierarchy ($k$-WL) briefly, and include it in our comparisons for reference. 

There are various further works that develop more expressive GNN variants, e.g. by extending the graph with random node features or IDs \cite{randomFeatures2, limits3, randomFeatures1}, or port numbers for edges \cite{ports}. However, these approaches also lose one of the fundamental properties of standard GNNs, namely permutation-equivariance: if the nodes are presented in a different order, then the assignment of features/ports (and hence the final embeddings) might also be different. As such, these GNNs either require training over an unreasonably large sample (observing e.g.\ all possible ID assignments to nodes), or they might not generalize so well to new data.

More detailed surveys of different GNN extensions are available in \cite{survey1, survey2}.

\section{Standard GNNs} \label{sec:GNN}

\subsection{Graphs and GNNs}

Our GNNs always operate on a simple undirected graph $G$. The number of nodes in $G$ is denoted by $n$, the neighbors of a node $u$ by $\text{N}(u)$. Our graph $G$ is potentially also equipped with a vector of features for each node; however, the hardest cases for distinguishing two graphs are usually when each node begins with identical features or no features at all.

Under the $d$-hop neighborhood of a node $u$, we understand the subgraph that $u$ can access in $d$ synchronous rounds of message passing: that is, all nodes at distance at most $d$ from $u$, and all edges where at least one endpoint is at distance at most $(d-1)$ from $u$.

Most state-of-the-art GNNs are so-called \textit{Message Passing Neural Networks}; we will refer to these as standard GNNs. Such a GNN begins with the node features as the initial embedding $h_u\!^{(0)}$ of a node $u$, and it operates in synchronous rounds. In each round $t$, every node $u$ computes a new embedding from (i) its own current embedding and (ii) the multiset of embeddings in its neighborhood; formally, the GNN is defined by the functions
\[ a_u\!^{(t)} = \textsc{aggregate} \: ( \, \{ \! \{  h_v\!^{(t-1)} \, | \, v \in \text{N}(u) \} \! \} \, ), \]
\[ h_u\!^{(t)} = \textsc{update} \: ( \, h_u\!^{(t-1)}, \, a_u\!^{(t)} \, ) \, , \]
where \textsc{aggregate} is a permutation-invariant function.

We assume that the GNN executes $d$ rounds of message passing (also called \textit{layers}), where $d$ is a small constant value in most cases. This produces a final embedding $h_u\!^{(d)}$ for each node $u$. In case of a graph classification task, the final embeddings of the nodes are also combined with a further \textsc{readout} function in the end to obtain an embedding that represents the entire graph.

When presenting our results in the paper, we mostly take a node classification perspective: that is, we consider a specific node $u$ in the graph, and the $d$-hop neighborhood of $u$. We say that a GNN variant can \textit{distinguish} two $d$-hop neighborhoods if there exists a realization of the GNN that computes a different final embedding for $u$ in the two cases.
However, most of our results also carry over to a graph classification setting: whenever two $d$-hop neighborhoods are indistinguishable in our examples, then the entire graphs are also indistinguishable by the given GNN model.

\subsection{The limits of GNNs}

It is known that the expressive power of message passing GNNs is upper bounded by the so-called \textit{Weisfeiler-Leman test}, also known as the WL test or color refinement algorithm \cite{gin, morris2019weisfeiler}.

The WL algorithm is a heuristic for isomorphism testing, where nodes in a graph are colored according to their features initially. Then in each iteration, this node coloring is further refined: every node uses a hash function to select a new color based on its current color and the multiset of colors in its immediate neighborhood. The refinement process stops whenever the number of different colors in the graph does not increase anymore.

Let us consider the example graphs on Figure \ref{fig:WLcounter}, which are clearly not isomorphic: e.g.\ one of them contains a triangle, while the other one does not. However, the graphs are not distinguishable by the WL test: the final coloring after refinement is shown in the figure. This shows that the two graphs are also not distinguishable by a GNN: the gray and the white nodes in both graphs will always compute the same final embedding in any GNN realization. One can check that the corresponding nodes indeed observe the same tree representation of their respective graph in the two cases.

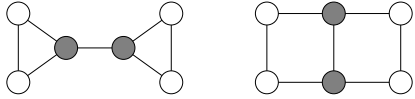
\begin{figure}
\centering
\resizebox{0.35\textwidth}{!}{

\begin{tikzpicture}

	\draw (0pt,0pt) -- (0pt,36pt);
	\draw (80pt,0pt) -- (80pt,36pt);
	\draw (25pt,18pt) -- (55pt,18pt);
	\draw (0pt,0pt) -- (25pt,18pt);
	\draw (0pt,36pt) -- (25pt,18pt);
	\draw (80pt,0pt) -- (55pt,18pt);
	\draw (80pt,36pt) -- (55pt,18pt);

	\draw[black, fill=white] (0pt,0pt) circle (6pt);
	\draw[black, fill=white] (0pt,36pt) circle (6pt);
	\draw[black, fill=gray] (25pt,18pt) circle (6pt);
	\draw[black, fill=gray] (55pt,18pt) circle (6pt);
	\draw[black, fill=white] (80pt,0pt) circle (6pt);
	\draw[black, fill=white] (80pt,36pt) circle (6pt);

	
	\draw (130pt,0pt) -- (130pt,36pt);
	\draw (165pt,0pt) -- (165pt,36pt);
	\draw (200pt,0pt) -- (200pt,36pt);
	\draw (130pt,0pt) -- (200pt,0pt);
	\draw (130pt,36pt) -- (200pt,36pt);
	
	\draw[black, fill=white] (130pt,0pt) circle (6pt);
	\draw[black, fill=white] (130pt,36pt) circle (6pt);
	\draw[black, fill=gray] (165pt,0pt) circle (6pt);
	\draw[black, fill=gray] (165pt,36pt) circle (6pt);
	\draw[black, fill=white] (200pt,0pt) circle (6pt);
	\draw[black, fill=white] (200pt,36pt) circle (6pt);

\end{tikzpicture}}
\caption{Graphs that are not distinguishable by $1$-WL, and hence also not by standard GNNs.}
\label{fig:WLcounter}
\end{figure}

It is also known that one can construct a sufficiently powerful GNN that has equivalent expressive power to the WL test, by devising an injective aggregate and update function \cite{gin}. This means that the GNN computes a different embedding for any two $d$-hop neighborhoods that can be separated by the WL test; intuitively speaking, such a GNN is as powerful as a general-purpose distributed algorithm in the same message passing model (i.e.\ without node IDs or port numbers).

\section{GNN Extensions Hierarchies} \label{sec:models}

We now define four different approaches for increasing the expressiveness of GNNs that we will study and compare in this paper. Due to space constraints, we only outline the main idea of these approaches here; we discuss their further advantages and drawbacks in more detail in Appendix \ref{app:models}.

\subsection{$k$-WL: The baseline hierarchy}

In most studies on the expressiveness of GNNs, the baseline hierarchy of expressive power is the so-called $k$-dimensional Weisfeiler-Leman algorithm. Note that there are two versions of this hierarchy in the literature, with slightly different indexing; here we consider the so-called ``Folklore'' indexing, sometimes also denoted by FWL.

In this hierarchy, $1$-WL simply corresponds to the WL test discussed before. For $k \geq 2$, a detailed description of the $k$-WL method is beyond the scope of this paper; intuitively speaking, the main idea behind the approach is to execute color refinement on the $k$-tuples of nodes in the original graph. This results in a framework with increasing expressive power: it is known that $(k+1)$-WL is always strictly more powerful than $k$-WL.

On the other hand, $k$-WL also has both time and space complexity that is lower bounded by $\Omega(n^k)$, i.e. it scales polynomially with the size of the entire graph; this heavily limits its usability in practice. In particular, most GNN extensions in the literature are only compared to $2$-WL, since it is already highly non-trivial to come up with graphs that are not distinguished by $2$-WL.

Recent works have also introduced GNNs variants based on these higher-order WL algorithms, which essentially inherit both the strengths and the weaknesses of $k$-WL \cite{maron2019provably, morris2019weisfeiler}.

\subsection{$S_k$: Counting substructures}

Since one of the most straightforward differences between the two graphs in Figure \ref{fig:WLcounter} is that only one of them has a triangle, it is a natural idea to directly extend our GNNs by subgraph counts up to a specific size $k$. This approach is most prominently applied in the work of \cite{bouritsas2020improving}, but it is also loosely connected to other GNN extensions \cite{Iaware, barcelo2021graph}.

\begin{figure*}
\centering
\hspace*{-0.01\textwidth}
\resizebox{1.02\textwidth}{!}{\input{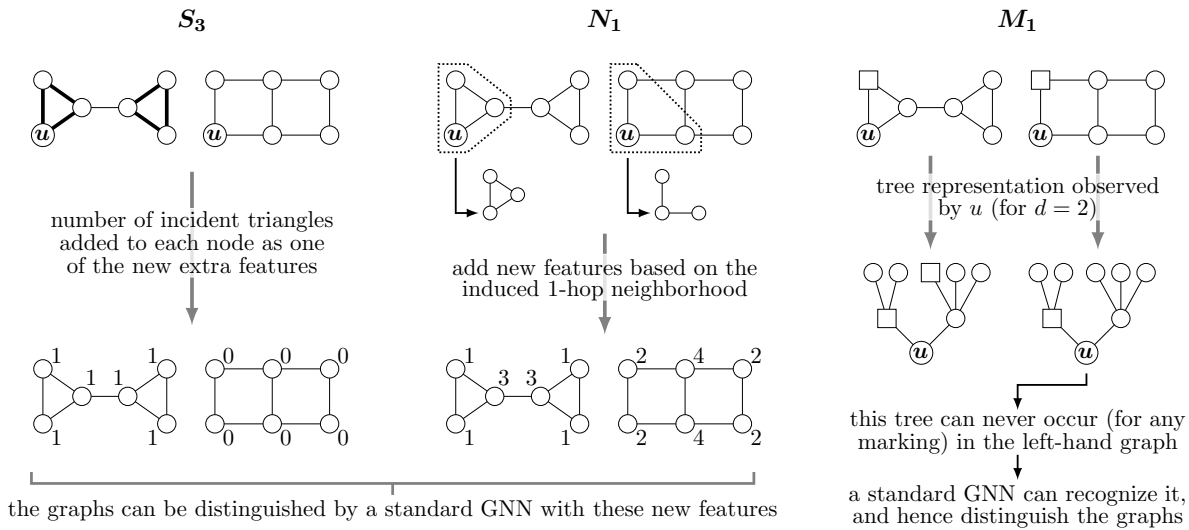}}
\caption{An illustration of how different extensions can distinguish the graphs in Figure \ref{fig:WLcounter}.}
\label{fig:examples}
\end{figure*}

In particular, let us define an $S_k$ GNN as follows: we assume that there is a preprocessing phase where for each $k' \in \{ 1, 2, ..., k \}$, we consider every different connected graph $G'$ on $k'$ nodes (up to isomorphism), and we count the number of times this graph $G'$ appears as an induced subgraph such that $u$ is one of the nodes of $G'$. We add these numbers as new features to each node $u$ in the graph, and then we run a standard GNN on the graph with these extended features.

Note that incident subgraphs of size $1$ and $2$ are also easy to compute in a standard GNN, so the method is only meaningful for $k \geq 3$. For $k=3$, the only two connected graphs on $3$ nodes are the triangle and the path of length $2$, so we add $2$ new features to each node before running a standard GNN.

We also point out that $S_k$ is the most expressive possible implementation of this subgraph-counting approach, since it considers all subgraphs of size up to $k$. In contrast to this, practical GNNs may only consider specific substructures (such as cliques or cycles), as the number of all non-isomorphic subgraphs increases rapidly as $k$ grows.

Finally, note that if $k \geq (d+2)$, then the newly added features may also contain information about a part of the graph that is not reachable by $u$ in the message passing phase. We will avoid this degenerate case, and only consider situations when any $k$-node subgraph is fully contained in the $d$-hop neighborhood of $u$; that is, we always ensure that either $k \leq (d+1)$, or the entire graph is contained within the $d$-hop neighborhood of $u$.

\subsection{$N_k$: Knowledge up to radius $k$}

Another similar approach is to not count the subgraphs incident to $u$, but to explicitly compute the isomorphism class of the $k$-hop induced neighborhood of $u$.

More formally, in the $N_k$ hierarchy, we assume that there is a mapping from all possible induced $k$-hop neighborhoods (that is, all graphs of radius at most $k$ up to isomorphism) to the real numbers, and each node $u$ is equipped with this number as an extra feature in a preprocessing step. This is then followed by regular message passing (i.e. a standard GNN) for $d$ rounds. As before, we will assume that $k < d$; otherwise, the message passing phase provides no extra information.

Note that for consistency with related work, our definition assumes that $N_k$ processes the \textit{induced} $k$-hop neighborhood (the graph induced by nodes at distance at most $k$ from $u$) instead of the $k$-hop neighborhood; that is, the preprocessing step is also aware of edges that have both endpoints at distance $k$ from $u$. For example, any triangle containing $u$ is entirely within the induced $1$-hop neighborhood of $u$. As such, the induced $1$-hop neighborhoods are different (for any node) in the two graphs of Figure \ref{fig:WLcounter}, so $N_1$ can already distinguish them from the new features.

Unless our graphs are very sparse, this GNN variant is not easy to apply in practice, since the preprocessing step already requires us indirectly to solve smaller instances of the graph isomorphism problem. Nonetheless, $N_k$ is still a valuable theoretical tool to study the power of GNNs when they are augmented by a complete understanding of the graph up to a small radius $k$.

The GNN variant of \cite{zhang2021nested} is directly based on a practical implementation of this idea; however, a similar use of ego networks also appears is more complex GNN extensions \cite{ego}.

\subsection{$M_k$: GNNs with markings}

Another approach to extend GNNs is to consider multiple, slightly perturbed variants of the input graph, and then use the collection of these to identify the original graph. That is, this model essentially executes multiple runs of a standard GNN on slightly changed variants of the graph, and in the end, it aggregates the final embeddings obtained in each run with a separate run-aggregation function.

One straightforward implementation of this idea is to remove some of the nodes (and their incident edges) from the graph in each run, and then execute message passing (i.e. a standard GNN) in the resulting network. A simple implementation of this idea with a randomized dropout of nodes is analyzed in \cite{nips}, whereas a more complex deterministic variant is discussed in \cite{reconstruction}. The approach also appears as one of the subcases in the framework of \cite{ESAN}.

Instead of directly studying this extension with node removals, we introduce a more general version of this idea, which we call \textit{GNNs with markings}. This extension inherits most of the properties of the node removal approach, but it has slightly larger expressive power. Intuitively, the main idea of markings is that the selected nodes are still distinguished from the remaining ones, but they are not removed from the graph; instead, the GNN is directly allowed to handle these nodes differently. We analyze GNNs with markings in more detail in Section \ref{sec:mark}.

If exactly $k$ nodes are marked in the $d$-hop neighborhood of $u$ in a run, then we refer to this run a $k$-marking of $u$. We define $M_k$ as the GNN which combines a standard GNN run over all distinct $k'$-markings of $u$, for every $k' \in \{ 0, ..., k\} $. That is, $M_k$ considers every version of the $d$-hop neighborhood around $u$ obtained by marking at most $k$ distinct nodes, computes an embedding for $u$ with the same GNN in each case, and then combines these into a final embedding for $u$.

In the example of Figure \ref{fig:WLcounter}, the two graphs can already be distinguished from a single marking (or alternatively, a single node removal). If $u$ is one of the nodes of degree $2$, then there will be a $1$-marking when its immediate neighbor of degree $2$ is marked (indicated by a square-shaped node in Figure \ref{fig:examples}). In left-hand graph, this also means that a node at distance $2$ from $u$ is also marked within the $d$-hop neighborhood of $u$, while in the right-hand graph, the nodes at distance $2$ will all be unmarked. These situations can all be recognized in a standard message passing phase, and thus $M_1$ separates the two graphs.

Note that the approach is most practical for small $k$ values, where the number of different $k$-markings is still relatively small.

\section{Discussion of GNNs with markings} \label{sec:mark}

As outlined before, GNNs with markings execute multiple runs of a standard GNN, with some of the nodes selected and marked in the beginning of each such run. The marked nodes are then treated differently from the rest: (i) in every round, nodes use a different function to aggregate from their marked an unmarked neighbors, and (ii) marked nodes also apply a different update function. 

That is, if $\text{N}_M(u)$ and $\text{N}_U(u)$ denote the marked and unmarked neighbors of $u$, respectively, then the new formula for message aggregation is
\begin{gather*}
a_u\!^{(t)} = \textsc{aggr}_{\text{marked}} \: ( \, \{ \! \{  h_v\!^{(t-1)} \, | \, v \in \text{N}_M(u) \} \! \} \, ) \, + \, \textsc{aggr}_{\text{unmarked}} \: ( \, \{ \! \{  h_v\!^{(t-1)} \, | \, v \in \text{N}_U(u) \} \! \} \, ) \, ,
\end{gather*}
where $\textsc{aggr}_{\text{marked}}$ and $\textsc{aggr}_{\text{unmarked}}$ are both permutation-invariant aggregation functions. Furthermore, marked and unmarked nodes learn a different update function ($\textsc{upd}_{\text{marked}}$ and $\textsc{upd}_{\text{unmarked}}$, respectively) to combine $h_u\!^{(t-1)}$ and $a_u\!^{(t)}$ into the new embedding $h_u\!^{(t)}$.

Similarly to GNNs with node removals, the final embeddings of $u$ in each run are combined in the end with a permutation-invariant run-aggregation function.

One can easily observe that the marking idea is a generalization of the node removal approach.

\begin{lemma}
Whenever two graphs $G_1$, $G_2$ are distinguishable by GNNs with node removals, they are also distinguishable by GNNs with markings.
\end{lemma}

\begin{proof}
Let us choose $\textsc{aggr}_{\text{marked}} \equiv 0$ and $\textsc{upd}_{\text{marked}} \equiv 0$. The resulting GNN with markings behaves as if the marked nodes were removed from the graph entirely.
\end{proof}

Furthermore, one can also show that this generalization is strict, i.e. GNNs with markings are strictly more expressive. Intuitively speaking, markings provide they same symmetry breaking opportunities for the GNN, but without an unnecessary loss of information. In particular, GNNs with node removals are unable to pass information through a missing node, and they are also unable to deduce whether two missing nodes were adjacent originally. The same problems do not appear in case of markings, where the underlying graph structure remains intact.

\begin{theorem} \label{th:markdrop}
There exists a pair of graphs that can be distinguished by GNNs with markings, but not by GNNs with node removals.
\end{theorem}

However, we note that our results in Section \ref{sec:compare} also carry over to the weaker model with node removals.

Finally, we prove that markings are indeed the most powerful GNNs that can be developed with this general symmetry-breaking idea, in similar sense as GINs were shown to be the most powerful standard GNNs \cite{gin}.

In order to characterize the maximal expressive power of this approach, we can again turn to the color refinement ($1$-WL) algorithm, but now from an initialization where marked nodes receive a different initial color than unmarked nodes. One can show by induction that if two nodes receive the same color (after $d$ iterations) in this algorithm, then a GNN with markings will compute the same final embedding for these two nodes. On the other hand, if we collect the final colors assigned by $1$-WL under the different markings, and this multiset is different for two nodes, then a sufficiently powerful GNN with markings can indeed distinguish the two cases.

More formally, let $G_1$ and $G_2$ be two $d$-hop neighborhoods around $u$, and let $\Gamma_1$ and $\Gamma_2$ denote the set of all possible markings (of at most $k$ nodes) in $G_1$ and $G_2$, respectively. We say that $m_1 \in \Gamma_1$ and $m_2 \in \Gamma_2$ are \textit{inseparable} markings of $G_1$ and $G_{2\,}$ if $1$-WL assigns the same color to $u$ in the two graphs when initialized according to $m_1$ and $m_2$. Finally, $G_1$ and $G_2$ are \textit{inseparable under} $k$-\textit{markings} if there is a bijection $\sigma: \Gamma_1 \rightarrow \Gamma_2$ such that $m_1$ and $\sigma(m_1)$ are inseparable for all $m_1 \in \Gamma_1$.

\begin{theorem} \label{th:inj}
There exists an injective implementation of GNNs with markings. That is, if $G_1$ and $G_2$ are not inseparable under $k$-markings, then the GNN computes a different final embedding for $u$ in the two graphs.
\end{theorem}

\section{Comparison of expressiveness} \label{sec:compare}

As our main result, we compare the expressive power of the different GNN extensions to each other in this section, with the details of the proofs discussed in Appendices \ref{app:CvC}$-$\ref{app:CFI}.

When comparing two GNN extensions $A$ and $B$, we will say that $A$ is more expressive than $B$ (denoted $A \succ\! B$) if there exists a pair of $d$-hop neighborhoods which can be distinguished by extension $A$, but not by extension $B$. Note that this not a strict ordering of extensions: in many cases, we will have both $A \succ\!B$ and $B \succ\! A$, i.e. both extensions can be superior to the other on different kinds of graphs. For strict superiority, we will use $A \subseteq B$ to show that any pair of graphs distinguishable by $A$ is also distinguishable by $B$.

A summary of our results is illustrated concisely in Figure \ref{fig:sum}. Note that there is an inherent offset in the indexing of the different hierarchies: recall that for the standard GNN model, we have $1$-WL $=$ $N_0$ $=$ $M_0$ $=$ $S_2$.

Finally, note that our constructions used in the proofs are all of reasonable size, in the sense that the number of nodes, edges and the maximal degree are all in $O(k)$.

\subsection{When $N_k$ and WL are superior} \label{sec:comp1}

We first present graph constructions where $N_k$ and the standard WL hierarchy outperform the remaining GNN extensions. We discuss $N_k$ and WL together because we can actually use the same construction to analyze their expressive power. 

\begin{theorem} \label{th:CvC}
For any $k \geq 1$, we have
\begin{itemize}[topsep=5pt,itemsep=0pt,partopsep=5pt,parsep=5pt]
    \item $N_1$ $\succ$ $S_k$ $\;$and$\;$ $N_1$ $\succ$ $M_k$,
    \item $2$-WL $\succ$ $S_k$ $\;$and$\;$ $2$-WL $\succ$ $M_k$.
\end{itemize}
\end{theorem}

\renewcommand*{\proofname}{Proof sketch}

\begin{proof}
Our proof is based on a generalization of the method used by \cite{nips} to show that not every graph can be distinguished with $2$ removed nodes.

Let us define two graphs $C_{\ell, \ell}$ and $C_{2\ell}$ as follows: for some parameter $\ell \geq 3$, let $C_{\ell, \ell}$ consist of two disjoint cycles of length $\ell$, and let $C_{2\ell}$ consist of a single cycle of length $2 \ell$. Finally, in both graphs, we add a single node $u$, and connect it to all the $2 \ell$ nodes in the graph (see Figure \ref{fig:CvC}).

Intuitively, the main idea of the proof is that our GNNs do not have enough rounds to go around even in the smaller cycle in $C_{\ell, \ell}$, hence they only observe smaller arcs of the cycles, which look identical in the two graphs. For standard a GNN, the tree representations are the same for each node in the two graphs. However, if $\ell \approx 2 \cdot k$, the two graphs cannot be distinguished by $S_k$ or $M_k$ either. In case of $S_k$, the subgraphs are not large enough to contain an entire $\ell$-cycle, only arcs of length at most $(k-1)$; the subgraph counts for such arcs (and their combinations) turn out to be identical in the two graphs. In case of $M_k$, we also find that a GNN essentially needs to mark nodes all around an $\ell$-cycle in order to recognize that this is a full cycle of length $\ell$, and not a smaller arc of length $\ell$ within a $2 \ell$-cycle.

On the other hand, the two graphs can be distinguished by $N_1$ for any $\ell$: since the entire graph is in the induced $1$-hop neighborhood of $u$ and the two graphs are non-isomorphic, $u$ will receive a different extra feature with $N_1$ in the two graphs, which makes them distinguishable. Similarly, one can show that $2$-WL can distinguish the two graphs.
\end{proof}

Furthermore, on two simpler constructions, one can also show that $N_1$ and $2$-WL can also be superior to each other (in fact, the entire other hierarchy) in some cases.

\begin{theorem} \label{th:NWL}
For any $k \geq 1$, we have $N_1$ $\succ$ $k$-WL.
\end{theorem}

\begin{theorem} \label{th:cycfirst}
For any $k \geq 1$, we have $2$-WL $\succ$ $N_k$.
\end{theorem}

\begin{figure}
\centering
\hspace{0.005\textwidth}
\minipage{0.41\textwidth}
\centering
    \vspace{29pt}
    \resizebox{1.0\textwidth}{!}{\definecolor{dgray}{gray}{0.35}

\begin{tikzpicture}

	\draw (0pt,0pt) -- (0pt,30pt);
	\draw (30pt,0pt) -- (30pt,30pt);
	\draw (60pt,0pt) -- (60pt,30pt);
	\draw (90pt,0pt) -- (90pt,30pt);
	\draw (0pt,0pt) -- (30pt,0pt);
	\draw (0pt,30pt) -- (30pt,30pt);
	\draw (60pt,0pt) -- (90pt,0pt);
	\draw (60pt,30pt) -- (90pt,30pt);
	
	\draw[dgray] (0pt,0pt) -- (45pt,-35pt);
	\draw[dgray] (30pt,0pt) -- (45pt,-35pt);
	\draw[dgray] (60pt,0pt) -- (45pt,-35pt);
	\draw[dgray] (90pt,0pt) -- (45pt,-35pt);
	\draw[dgray] (0pt,30pt) -- (45pt,-35pt);
	\draw[dgray] (30pt,30pt) -- (45pt,-35pt);
	\draw[dgray] (60pt,30pt) -- (45pt,-35pt);
	\draw[dgray] (90pt,30pt) -- (45pt,-35pt);

	\draw[black, fill=white] (0pt,0pt) circle (4pt);
	\draw[black, fill=white] (0pt,30pt) circle (4pt);
	\draw[black, fill=white] (30pt,0pt) circle (4pt);
	\draw[black, fill=white] (30pt,30pt) circle (4pt);
	\draw[black, fill=white] (60pt,0pt) circle (4pt);
	\draw[black, fill=white] (60pt,30pt) circle (4pt);
	\draw[black, fill=white] (90pt,0pt) circle (4pt);
	\draw[black, fill=white] (90pt,30pt) circle (4pt);
	
	\draw[black, fill=white] (45pt,-35pt) circle (4pt);
	
	
	\draw (170pt,0pt) -- (170pt,30pt);
	\draw (260pt,0pt) -- (260pt,30pt);
	\draw (170pt,0pt) -- (260pt,0pt);
	\draw (170pt,30pt) -- (260pt,30pt);
	
	\draw[dgray] (170pt,0pt) -- (215pt,-35pt);
	\draw[dgray] (200pt,0pt) -- (215pt,-35pt);
	\draw[dgray] (230pt,0pt) -- (215pt,-35pt);
	\draw[dgray] (260pt,0pt) -- (215pt,-35pt);
	\draw[dgray] (170pt,30pt) -- (215pt,-35pt);
	\draw[dgray] (200pt,30pt) -- (215pt,-35pt);
	\draw[dgray] (230pt,30pt) -- (215pt,-35pt);
	\draw[dgray] (260pt,30pt) -- (215pt,-35pt);

	\draw[black, fill=white] (170pt,0pt) circle (4pt);
	\draw[black, fill=white] (170pt,30pt) circle (4pt);
	\draw[black, fill=white] (200pt,0pt) circle (4pt);
	\draw[black, fill=white] (200pt,30pt) circle (4pt);
	\draw[black, fill=white] (230pt,0pt) circle (4pt);
	\draw[black, fill=white] (230pt,30pt) circle (4pt);
	\draw[black, fill=white] (260pt,0pt) circle (4pt);
	\draw[black, fill=white] (260pt,30pt) circle (4pt);
	
	\draw[black, fill=white] (215pt,-35pt) circle (4pt);
	
\end{tikzpicture}}
    \vspace{-6pt}
    \caption{Illustration of the $C_{\ell, \ell}$ and $C_{2\ell}$ graphs for $\ell=4$.}
    \label{fig:CvC}
\endminipage\hfill
\hspace{0.1\textwidth}
\minipage{0.45\textwidth}
\centering
    \resizebox{1.0\textwidth}{!}{\begin{tikzpicture}

	\draw (0pt,0pt) -- (0pt,90pt);
	\draw (30pt,0pt) -- (30pt,90pt);
	\draw (60pt,0pt) -- (60pt,90pt);
	\draw (90pt,0pt) -- (90pt,90pt);
	\draw (0pt,0pt) -- (90pt,0pt);
	\draw (0pt,30pt) -- (90pt,30pt);
	\draw (0pt,60pt) -- (90pt,60pt);
	\draw (0pt,90pt) -- (90pt,90pt);
	
	\draw (0pt,0pt) arc (216.87:143.13:75pt);
	\draw (0pt,0pt) arc (216.87:143.13:50pt);
	\draw (0pt,30pt) arc (216.87:143.13:50pt);
	
	\draw (30pt,0pt) arc (216.87:143.13:75pt);
	\draw (30pt,0pt) arc (216.87:143.13:50pt);
	\draw (30pt,30pt) arc (216.87:143.13:50pt);
	
	\draw (60pt,0pt) arc (-36.87:36.87:75pt);
	\draw (60pt,0pt) arc (-36.87:36.87:50pt);
	\draw (60pt,30pt) arc (-36.87:36.87:50pt);
	
	\draw (90pt,0pt) arc (-36.87:36.87:75pt);
	\draw (90pt,0pt) arc (-36.87:36.87:50pt);
	\draw (90pt,30pt) arc (-36.87:36.87:50pt);
	
	\draw (0pt,0pt) arc (-126.87:-53.13:75pt);
	\draw (0pt,0pt) arc (-126.87:-53.13:50pt);
	\draw (30pt,0pt) arc (-126.87:-53.13:50pt);
	
	\draw (0pt,30pt) arc (-126.87:-53.13:75pt);
	\draw (0pt,30pt) arc (-126.87:-53.13:50pt);
	\draw (30pt,30pt) arc (-126.87:-53.13:50pt);
	
	\draw (0pt,60pt) arc (126.87:53.13:75pt);
	\draw (0pt,60pt) arc (126.87:53.13:50pt);
	\draw (30pt,60pt) arc (126.87:53.13:50pt);
	
	\draw (0pt,90pt) arc (126.87:53.13:75pt);
	\draw (0pt,90pt) arc (126.87:53.13:50pt);
	\draw (30pt,90pt) arc (126.87:53.13:50pt);

	\draw[black, fill=white] (0pt,0pt) circle (3pt);
	\draw[black, fill=white] (0pt,30pt) circle (3pt);
	\draw[black, fill=white] (0pt,60pt) circle (3pt);
	\draw[black, fill=white] (0pt,90pt) circle (3pt);
	\draw[black, fill=white] (30pt,0pt) circle (3pt);
	\draw[black, fill=white] (30pt,30pt) circle (3pt);
	\draw[black, fill=white] (30pt,60pt) circle (3pt);
	\draw[black, fill=white] (30pt,90pt) circle (3pt);
	\draw[black, fill=white] (60pt,0pt) circle (3pt);
	\draw[black, fill=white] (60pt,30pt) circle (3pt);
	\draw[black, fill=white] (60pt,60pt) circle (3pt);
	\draw[black, fill=white] (60pt,90pt) circle (3pt);
	\draw[black, fill=white] (90pt,0pt) circle (3pt);
	\draw[black, fill=white] (90pt,30pt) circle (3pt);
	\draw[black, fill=white] (90pt,60pt) circle (3pt);
	\draw[black, fill=white] (90pt,90pt) circle (3pt);

	
	\draw (170pt,0pt) -- (170pt,90pt);
	\draw (200pt,0pt) -- (200pt,90pt);
	\draw (230pt,0pt) -- (230pt,90pt);
	\draw (260pt,0pt) -- (260pt,90pt);
	\draw (170pt,0pt) -- (260pt,0pt);
	\draw (170pt,30pt) -- (260pt,30pt);
	\draw (170pt,60pt) -- (260pt,60pt);
	\draw (170pt,90pt) -- (260pt,90pt);
	
	\draw (170pt,30pt) -- (200pt,0pt);
	\draw (170pt,60pt) -- (230pt,0pt);
	\draw (170pt,90pt) -- (260pt,0pt);
	\draw (200pt,90pt) -- (260pt,30pt);
	\draw (230pt,90pt) -- (260pt,60pt);
	
	\draw (170pt,0pt) -- (200pt,90pt);
	\draw (200pt,0pt) -- (230pt,90pt);
	\draw (230pt,0pt) -- (260pt,90pt);
	\draw (170pt,0pt) -- (260pt,30pt);
	\draw (170pt,30pt) -- (260pt,60pt);
	\draw (170pt,60pt) -- (260pt,90pt);
	
	\draw (170pt,0pt) arc (216.87:143.13:75pt);
	\draw (260pt,0pt) arc (-36.87:36.87:75pt);
	\draw (170pt,0pt) arc (-126.87:-53.13:75pt);
	\draw (170pt,90pt) arc (126.87:53.13:75pt);
	
	\draw (200pt,0pt) arc (216.87:143.13:75pt);
	\draw (230pt,0pt) arc (-36.87:36.87:75pt);
	\draw (170pt,30pt) arc (-126.87:-53.13:75pt);
	\draw (170pt,60pt) arc (126.87:53.13:75pt);
	
	\draw (170pt,90pt) arc (68.2:21.8:161.55pt);
	

	\draw[black, fill=white] (170pt,0pt) circle (3pt);
	\draw[black, fill=white] (170pt,30pt) circle (3pt);
	\draw[black, fill=white] (170pt,60pt) circle (3pt);
	\draw[black, fill=white] (170pt,90pt) circle (3pt);
	\draw[black, fill=white] (200pt,0pt) circle (3pt);
	\draw[black, fill=white] (200pt,30pt) circle (3pt);
	\draw[black, fill=white] (200pt,60pt) circle (3pt);
	\draw[black, fill=white] (200pt,90pt) circle (3pt);
	\draw[black, fill=white] (230pt,0pt) circle (3pt);
	\draw[black, fill=white] (230pt,30pt) circle (3pt);
	\draw[black, fill=white] (230pt,60pt) circle (3pt);
	\draw[black, fill=white] (230pt,90pt) circle (3pt);
	\draw[black, fill=white] (260pt,0pt) circle (3pt);
	\draw[black, fill=white] (260pt,30pt) circle (3pt);
	\draw[black, fill=white] (260pt,60pt) circle (3pt);
	\draw[black, fill=white] (260pt,90pt) circle (3pt);

\end{tikzpicture}}
    \vspace{-15pt}
    \caption{The Rook's $4x4$ and Shrikhande graphs.}
    \label{fig:rook}
\endminipage\hfill
\hspace{0.01\textwidth}
\end{figure}

\subsection{When $M_k$ is superior}

In a similar fashion, we show that GNNs with markings can outperform the other methods in terms of expressiveness.

\begin{theorem} \label{th:markcycle}
For any $k \geq 1$, we have $M_1$ $\succ$ $N_k$ and $M_1$ $\succ$ $S_k$.
\end{theorem}

Furthermore, two markings (or node removals) are also sufficient to make a GNN more expressive than $2$-WL.

\begin{theorem} \label{th:rook}
We have $M_2$ $\succ$ $2$-WL.
\end{theorem}

\begin{proof}
One can show this through the most popular example graphs for indistinguishability with $2$-WL: the Rook's $4x4$ and Shrikhande graphs, shown in Figure \ref{fig:rook}.

Let $u$ be any node in these graphs, and consider $M_2$ with $d=2$. Note that the induced $1$-hop neighborhood of $u$ in the two graphs is identical to $C_{3, 3}$ and $C_{6}$, respectively. Consider a $2$-marking where the marked nodes are adjacent to both $u$ and to each other. In this case, in the Rook's graph, there will also be a neighbor of $u$ which has two marked neighbors, whereas in the Shrikhande graph, $u$ will have no such neighbor.

That is, if there is a run where (i) $u$ has two marked neighbors, (ii) both marked neighbors have a marked neighbor, and (iii) $u$ has an unmarked neighbor with two marked neighbors, then it can deduce that it is in the Rook's graph. All these properties can be verified by a GNN with $d=2$.
\end{proof}

\subsection{When $S_k$ is superior} \label{sec:comp3}

Finally, let us consider graphs that can be distinguished by $S_k$, but not by the other methods. We note that some of these results require slight adjustments to carry over to the graph classification setting; see Appendix \ref{app:global} for details.

The comparison of $S_k$ and WL has already been conducted by Bouristas \textit{et al.} \cite{bouritsas2020improving}, who show that $S_4$ can already be superior to $2$-WL. In particular, the graphs in Figure \ref{fig:rook} can be distinguished based on the incident number of $4$-cliques, but cannot be separated by $2$-WL. We also add this as an explicit theorem for completeness.

\begin{theorem} \label{th:GSN}
We have $S_4$ $\succ$ $2$-WL.
\end{theorem}

On the other hand, it is also known that $2$-WL can count triangles and paths of length $2$, so $S_3$ $\subseteq$ $2$-WL.

As for the relationship of $S_k$ to the remaining methods, we cannot prove a result as general as in the previous cases, i.e. that $S_k$ for a certain $k$ is stronger than the rest of the extensions for any parameter $j$. For example, in case of $N_j$, it is already somewhat clear intuitively that $S_k$ can only hope to be better than $N_j$ as long as it has extra information, i.e. the counted substructures are not entirely contained in the induced $j$-hop neighborhood of $u$.

\begin{theorem} \label{th:SWL}
For any $k \geq 0$, we have
\begin{itemize}[topsep=5pt,itemsep=5pt,partopsep=5pt,parsep=0pt]
    \item $S_{(k+2)}$ $\subseteq$ $N_k$,
    \item $S_{(k+3)}$ $\succ$ $N_k$.
\end{itemize}
\end{theorem}

\begin{proof}
To show the containment result, note that except for the path of length $(k+1)$, any other graph on $(k+2)$ nodes has radius at most $k$. Hence if such a subgraph is incident to $u$, then it is contained entirely in the induced $k$-hop neighborhood; thus all these subgraphs can also be counted by $N_k$. As for paths of length $(k+1)$, the first $k$ edges of such paths are also always contained in the $k$-hop neighborhood, and the number of potential edges to conclude such a path can be inferred from the degree of the penultimate node (at distance $k$) by a standard GNN.

To show $S_{(k+3)}$ $\succ$ $N_k$, we can consider a path of length $k$ from $u$ to another node $v$. In $G_1$, we add another triangle incident to $v$, whereas in $G_2$, we add two more outgoing paths of length $2$; these seem identical to a standard GNN with $d=k+2$, and the induced neighborhoods are also identical up to $k$ hops. However, the entire graph in $G_1$ (the path and the triangle) consists of only $(k+3)$ nodes, so $S_{(k+3)}$ can use it to distinguish the two graphs.
\end{proof}

Finally, the most challenging task is to compare $S_k$ to the $M_j$ hierarchy. Here our results are not necessarily tight: for smaller $k$ values, it remains an open question whether $S_k$ can still outperform markings, or if a GNN with markings can indirectly infer the number of substructures. We will briefly revisit this question in Section \ref{sec:count}.

\begin{theorem} \label{th:CvC2}
For any $k \geq 1$, we have $S_{(2k+2)}$ $\succ$ $M_k$.
\end{theorem}

\begin{proof}
One can prove this through a more rigorous analysis of the $C_{\ell, \ell}$ vs. $C_{2\ell}$ construction with a choice of $\ell=2k+1$. Intuitively, one can show that $u$ cannot distinguish the two graphs unless it can mark every second node in one of the small cycles of $C_{\ell, \ell}$, i.e. at least $k \geq \frac{\ell}{2}$ nodes. As such, $M_k$ cannot separate the two cases for $\ell=2k+1$. On the other hand, $S_{(\ell+1)}$ can distinguish the graphs from the $(\ell+1)$-node subgraph formed by one of the $\ell$-cycles and $u$, which only appears in $C_{\ell, \ell}$.
\end{proof}

\subsection{The limits of each approach}

Finally, we point out that while our GNN extensions are rather powerful, they are still far from efficiently distinguishing any pair of graphs. Some of our previous constructions can already be used to show that there are graphs that remain indistinguishable to our extensions until the parameter $k$ is only an additive/multiplicative constant away from the size of the graph. Moreover, one can essentially combine these properties in a single example which is simultaneously challenging for all of our extensions.

\begin{theorem} \label{th:CFI}
For all of $S_k$, $N_k$ and $M_k$, there exist a construction of pairs of non-isomorphic graphs $(G_1, G_1'), \, (G_2, G_2'), \, ... \,$ (of increasing size $n_i=|G_i|=|G_i'|$), such that $G_i$ and $G_i'$ in the corresponding construction cannot be distinguished
\begin{itemize}[topsep=5pt,itemsep=0pt,partopsep=5pt,parsep=5pt]
    \item with $S_k$ unless $k \geq n_i - O(1)$,
    \item with $N_k$ unless $k \geq n_i - O(1)$,
    \item with $M_k$ unless $k \geq n_i \, / \, O(1)$.
\end{itemize}
\end{theorem}

\renewcommand*{\proofname}{Proof}

\section{Counting cliques and cycles} \label{sec:count}

Besides these direct comparisons, another natural way to evaluate GNN variants is by their ability to count specific small substructures in the graph. Two of the most natural choices for substructures of interest are cliques and cycles, which are known to be very relevant for applications in social science and molecule recognition, respectively \cite{subsapp2, subsapp1}. Previous work has already studied the standard $k$-WL hierarchy extensively in terms of its ability to count cliques and cycles \cite{WLcount, subcount, cellular}.

As such, we now also study whether it is possible to compute the number of $\ell$-cliques or induced $\ell$-cycles with our improved GNN variants for some $\ell \geq 3$, with the goal of finding the highest $\ell$ for which this is possible. Recall that $1$-WL is not even able to count $3$-cliques or $3$-cycles (i.e. triangles). We again focus on the problem from a single node's perspective (whether $u$ is able to count the number of $\ell$-cliques or $\ell$-cycles it is contained in), but the results also carry over to the global problem of counting the total number of such structures in the graph.

The question is easiest to answer for the $S_k$ hierarchy, where nodes are directly provided with the number of all incident substructures of size $k$. On the other hand, one can show that $S_k$ is unable to count larger structures than those that are already counted in its preprocessing phase.

\begin{theorem} \label{th:count_S}
An $S_k$ GNN can count $k$-cliques and $k$-cycles, but it cannot count $(k+1)$-cycles.
\end{theorem}

We also show that the result is also tight on cliques for small $k$ values, i.e. that $S_k$ cannot count $(k+1)$-cliques.

In case of $N_k$, counting cliques is straightforward, since any clique is already contained in the $1$-hop induced neighborhood of a node. As for cycles, $N_k$ can only count them until we can ensure that the entire cycle is within the $k$-hop induced neighborhood of $u$.

\begin{theorem} \label{th:count_N}
An $N_1$ GNN can count $\ell$-cliques for any $\ell \geq 3$. An $N_k$ GNN can count $(2k+1)$-cycles, but it cannot count $(2k+2)$-cycles.
\end{theorem}

Finally, $M_k$ turns out to be more challenging to analyze. However, it is still relatively straightforward to show that the set of $k$-markings allow us to identify all $(k+2)$-cliques in the graph.

\begin{theorem} \label{th:count_M}
An $M_k$ GNN can count $(k+2)$-cliques.
\end{theorem}

We again show that this result is also tight for small $k$ values, i.e. that $M_k$ cannot count $(k+3)$-cliques. Counting cycles with $M_k$, on the other hand, is a much harder problem; we discuss some results on how $M_k$ can count cycles in a more limited sense in Appendix \ref{app:counting}.

\bibliography{references}

\begin{thebibliography}{10}

\bibitem{randomFeatures2}
Ralph Abboud, Ismail~Ilkan Ceylan, Martin Grohe, and Thomas Lukasiewicz.
\newblock The surprising power of graph neural networks with random node
  initialization.
\newblock In {\em International Joint Conference on Artificial Intelligence
  ({IJCAI-21})}, pages 2112--2118, 2021.

\bibitem{WLcount}
V.~Arvind, Frank Fuhlbrück, Johannes Köbler, and Oleg Verbitsky.
\newblock On weisfeiler-leman invariance: Subgraph counts and related graph
  properties.
\newblock {\em Journal of Computer and System Sciences}, 113:42--59, 2020.

\bibitem{barcelo2021graph}
Pablo Barcel{\'o}, Floris Geerts, Juan Reutter, and Maksimilian Ryschkov.
\newblock Graph neural networks with local graph parameters.
\newblock {\em arXiv preprint arXiv:2106.06707}, 2021.

\bibitem{ESAN}
Beatrice Bevilacqua, Fabrizio Frasca, Derek Lim, Balasubramaniam Srinivasan,
  Chen Cai, Gopinath Balamurugan, Michael~M Bronstein, and Haggai Maron.
\newblock Equivariant subgraph aggregation networks.
\newblock {\em arXiv preprint arXiv:2110.02910}, 2021.

\bibitem{cellular}
Cristian Bodnar, Fabrizio Frasca, Nina Otter, Yu~Guang Wang, Pietro Li{\`{o}},
  Guido Mont{\'{u}}far, and Michael~M. Bronstein.
\newblock Weisfeiler and lehman go cellular: {CW} networks.
\newblock In {\em Advances in Neural Information Processing Systems
  ({NeurIPS})}, volume~34, 2021.

\bibitem{simplicial}
Cristian Bodnar, Fabrizio Frasca, Yuguang Wang, Nina Otter, Guido~F Montufar,
  Pietro Li{\'o}, and Michael Bronstein.
\newblock Weisfeiler and lehman go topological: Message passing simplicial
  networks.
\newblock In {\em International Conference on Machine Learning (ICML)}, volume
  139, pages 1026--1037, 2021.

\bibitem{bouritsas2020improving}
Giorgos Bouritsas, Fabrizio Frasca, Stefanos Zafeiriou, and Michael~M
  Bronstein.
\newblock Improving graph neural network expressivity via subgraph isomorphism
  counting.
\newblock {\em arXiv preprint arXiv:2006.09252}, 2020.

\bibitem{CFI}
Jin-Yi Cai, Martin F{\"u}rer, and Neil Immerman.
\newblock An optimal lower bound on the number of variables for graph
  identification.
\newblock {\em Combinatorica}, 12(4):389--410, 1992.

\bibitem{subcount}
Zhengdao Chen, Lei Chen, Soledad Villar, and Joan Bruna.
\newblock Can graph neural networks count substructures?
\newblock In {\em Advances in Neural Information Processing Systems
  ({NeurIPS})}, volume~33, pages 10383--10395, 2020.

\bibitem{reconstruction}
Leonardo Cotta, Christopher Morris, and Bruno Ribeiro.
\newblock Reconstruction for powerful graph representations.
\newblock In {\em Advances in Neural Information Processing Systems
  ({NeurIPS})}, volume~34, 2021.

\bibitem{fout2017protein}
Alex Fout, Jonathon Byrd, Basir Shariat, and Asa Ben-Hur.
\newblock Protein interface prediction using graph convolutional networks.
\newblock In {\em Advances in Neural Information Processing Systems
  ({NeurIPS})}, volume~30, 2017.

\bibitem{limits}
Vikas Garg, Stefanie Jegelka, and Tommi Jaakkola.
\newblock Generalization and representational limits of graph neural networks.
\newblock In {\em International Conference on Machine Learning (ICML)}, volume
  119, pages 3419--3430, 2020.

\bibitem{gilmer2017neural}
Justin Gilmer, Samuel~S Schoenholz, Patrick~F Riley, Oriol Vinyals, and
  George~E Dahl.
\newblock Neural message passing for quantum chemistry.
\newblock In {\em International Conference on Machine Learning (ICML)}, August
  2017.

\bibitem{subsapp1}
M.~Girvan and M.~E.~J. Newman.
\newblock Community structure in social and biological networks.
\newblock {\em Proceedings of the National Academy of Sciences},
  99(12):7821--7826, 2002.

\bibitem{kiefer2019power}
Sandra Kiefer and Daniel Neuen.
\newblock {The Power of the Weisfeiler-Leman Algorithm to Decompose Graphs}.
\newblock In {\em 44th International Symposium on Mathematical Foundations of
  Computer Science (MFCS 2019)}, volume 138 of {\em LIPIcs}, pages 45:1--45:15,
  2019.

\bibitem{limits2}
Andreas Loukas.
\newblock How hard is to distinguish graphs with graph neural networks?
\newblock In {\em Advances in Neural Information Processing Systems
  ({NeurIPS})}, volume~33, pages 3465--3476, 2020.

\bibitem{limits3}
Andreas Loukas.
\newblock What graph neural networks cannot learn: depth vs width.
\newblock In {\em 8th International Conference on Learning Representations
  ({ICLR})}, 2020.

\bibitem{maron2019provably}
Haggai Maron, Heli Ben-Hamu, Hadar Serviansky, and Yaron Lipman.
\newblock Provably powerful graph networks.
\newblock In {\em Advances in Neural Information Processing Systems
  ({NeurIPS})}, volume~32, 2019.

\bibitem{subsapp2}
R.~Milo, S.~Shen-Orr, S.~Itzkovitz, N.~Kashtan, D.~Chklovskii, and U.~Alon.
\newblock Network motifs: Simple building blocks of complex networks.
\newblock {\em Science}, 298(5594):824--827, 2002.

\bibitem{survey2}
Christopher Morris, Yaron Lipman, Haggai Maron, Bastian Rieck, Nils~M Kriege,
  Martin Grohe, Matthias Fey, and Karsten Borgwardt.
\newblock Weisfeiler and leman go machine learning: The story so far.
\newblock {\em arXiv preprint arXiv:2112.09992}, 2021.

\bibitem{morris2019weisfeiler}
Christopher Morris, Martin Ritzert, Matthias Fey, William~L Hamilton, Jan~Eric
  Lenssen, Gaurav Rattan, and Martin Grohe.
\newblock Weisfeiler and {L}eman go neural: Higher-order graph neural networks.
\newblock In {\em Proceedings of the AAAI Conference on Artificial
  Intelligence}, volume~33, pages 4602--4609, 2019.

\bibitem{nips}
P\'al~Andr\'as Papp, Karolis Martinkus, Lukas Faber, and Roger Wattenhofer.
\newblock Drop{GNN}: Random dropouts increase the expressiveness of graph
  neural networks.
\newblock In {\em Advances in Neural Information Processing Systems
  ({NeurIPS})}, volume~34, 2021.

\bibitem{sanchez2020learning}
Alvaro Sanchez-Gonzalez, Jonathan Godwin, Tobias Pfaff, Rex Ying, Jure
  Leskovec, and Peter Battaglia.
\newblock Learning to simulate complex physics with graph networks.
\newblock In {\em International Conference on Machine Learning (ICML)}, pages
  8459--8468, 2020.

\bibitem{ego}
Dylan Sandfelder, Priyesh Vijayan, and William~L. Hamilton.
\newblock Ego-{GNN}s: Exploiting ego structures in graph neural networks.
\newblock In {\em IEEE International Conference on Acoustics, Speech and Signal
  Processing ({ICASSP})}, pages 8523--8527, 2021.

\bibitem{survey1}
Ryoma Sato.
\newblock A survey on the expressive power of graph neural networks.
\newblock {\em arXiv preprint arXiv:2003.04078}, 2020.

\bibitem{ports}
Ryoma Sato, Makoto Yamada, and Hisashi Kashima.
\newblock Approximation ratios of graph neural networks for combinatorial
  problems.
\newblock In {\em Advances in Neural Information Processing Systems
  ({NeurIPS})}, volume~32, 2019.

\bibitem{randomFeatures1}
Ryoma Sato, Makoto Yamada, and Hisashi Kashima.
\newblock Random features strengthen graph neural networks.
\newblock In {\em Proceedings of the 2021 SIAM International Conference on Data
  Mining (SDM)}, pages 333--341, 2021.

\bibitem{scarselli2008graph}
Franco Scarselli, Marco Gori, Ah~Chung Tsoi, Markus Hagenbuchner, and Gabriele
  Monfardini.
\newblock The graph neural network model.
\newblock {\em IEEE Transactions on Neural Networks}, 2008.

\bibitem{wu2020comprehensive}
Zonghan Wu, Shirui Pan, Fengwen Chen, Guodong Long, Chengqi Zhang, and S~Yu
  Philip.
\newblock A comprehensive survey on graph neural networks.
\newblock {\em IEEE Transactions on Neural Networks and Learning Systems},
  2020.

\bibitem{gin}
Keyulu Xu, Weihua Hu, Jure Leskovec, and Stefanie Jegelka.
\newblock How powerful are graph neural networks?
\newblock In {\em International Conference on Learning Representations
  ({ICLR})}, 2019.

\bibitem{ying2018graph}
Rex Ying, Ruining He, Kaifeng Chen, Pong Eksombatchai, William~L Hamilton, and
  Jure Leskovec.
\newblock Graph convolutional neural networks for web-scale recommender
  systems.
\newblock In {\em Proceedings of the 24th ACM SIGKDD International Conference
  on Knowledge Discovery \& Data Mining}, pages 974--983, 2018.

\bibitem{Iaware}
Jiaxuan You, Jonathan~M Gomes-Selman, Rex Ying, and Jure Leskovec.
\newblock Identity-aware graph neural networks.
\newblock In {\em Proceedings of the AAAI Conference on Artificial
  Intelligence}, volume~35, pages 10737--10745, 2021.

\bibitem{zhang2021nested}
Muhan Zhang and Pan Li.
\newblock Nested graph neural networks.
\newblock In {\em Advances in Neural Information Processing Systems
  ({NeurIPS})}, volume~34, 2021.

\end{thebibliography}
\bibliographystyle{plain}

\newpage
\appendix
\onecolumn
\section{More details on the GNN extensions} \label{app:models}

In this section we discuss each of our GNN extensions in slightly more detail.

In case of $S_k$, a naive implementation of the method would require a preprocessing time that is in $O(\binom{|H|}{k})$, where $|H|$ denotes the number of nodes in the $(k-1)$-neighborhood of $u$. Let us denote the largest degree in the graph by $\Delta$; in several applications (e.g. chemistry or biology), this $\Delta$ is essentially considered a constant. In these cases, one can apply a naive upper bound of $|H| \leq O(\Delta^{k-1})$. As such, if both $\Delta$ and $k$ are small constants, then the required preprocessing time can be essentially linear in the size $n$ of the whole graph, which makes it much more efficient than higher-order WL methods. However, another practical issue with $S_k$ is the number of new features added, since the number of non-isomorphic connected graphs on $k$ nodes grows rapidly.

For $N_k$, the running time of the preprocessing phase again depends on the size of the $k$-hop neighborhood of $u$; we can again upper bound this by $O(\Delta^{k-1})$. Note that we need to compute the isomorphism class of the neighborhood in question; since the best known exact algorithms for isomorphism testing are only quasi-polynomial, this only makes the approach viable for very small $k$ values and on graphs with very small $\Delta$ in practice. Alternatively, one might use a hash function or a more sophisticated heuristic to approximate these extra features in a practical implementation. 

Note that for both $S_k$ and $N_k$, an alternative approach would be to assign the extra features to $u$ by not only distinguishing different subgraphs/induced neighborhoods, but also based on the position of $u$ within the given graph. That is, e.g. for a given subgraph of size at most $k$, we not only add a single extra feature, but multiple features also based on where $u$ is located in the given subgraph. Some of the related works that study these approaches also consider this idea. We note that this modification does not affect our results: our positive results on $S_k$ and $N_k$ do not require this information, while our claims on the limits of the methods still hold if this information is added.

In case of markings or node removals, if $|H|$ now denotes the size of the $d$-hop neighborhood of $u$, then the number of possible ways to mark at most $k$ nodes in this graph is in $O(\binom{|H|}{k})$; note that in contrast to $S_k$, $H$ may now be in the (possibly much larger) magnitude of $O(\Delta^{d})$. In each case, we now need a separate GNN run for $d$ rounds in the graph induced by the $d$-hop neighborhood of $u$, executing this for all $n$ nodes of the graph separately. An alternative is the probabilistic approach of \cite{nips}, which allows us to execute the method on the whole graph together, with all nodes observing the same markings in a single GNN run; however, in this case, we are only guaranteed to observe every possible combination of marked (or removed) nodes with a certain probability.

Another question regarding the marking/node removal approach is whether $u$ is in general aware of the number of nodes that are marked/removed in the current run. That is, for simplicity, we assumed that our framework is provided with each configuration that corresponds to a $k'$-marking (for $k' \in  \{ 0, ..., k \}$), it computes an embedding for $u$, and then it merges this multiset into a final embedding. Another possibility would be a setting where for each run, besides an embedding, $u$ is also directly informed of the number of nodes marked in the run (and hence formally the multiset becomes a multiset of pairs, each pair consisting of an embedding and an integer $k'$). The reason this needs to be discussed is because in the probabilistic setting, the GNN may encounter a specific $k'$-marking multiple times, and thus might be able to infer $k'$ from the frequency of these cases, hence obtaining extra information compared to the discussed case when it is directly ensured that each $k'$-marking is visited exactly once. However, we note that this change is not relevant for our proofs; that is, whenever we show that $M_k$ is more expressive than another model, the corresponding proofs do not require us to directly label the specific markings by the number of marked nodes, and whenever a pair of graphs cannot be distinguished by $M_k$, then the same proof still holds even if the specific runs are labeled with the corresponding $k'$ value.

We note that an alternative method is to perturb (e.g. remove) not nodes, but edges of the original; this has also been studied in \cite{ESAN}.

\section{GNNs with markings} \label{app:marking}

\subsection{Discussion of GNNs with markings}

Note that GNNs with markings are a very natural generalization of the node removal approach introduced in previous works: we still provide the same symmetry breaking information to the GNN in each run, but instead of removing the nodes, we allow the GNN to learn the way in which these selected nodes are handled differently.

Since the main idea for symmetry breaking is identical, many aspects from the analysis in these previous works also carry over to our model without changes, e.g. the number of possible $k$-markings in a $d$-hop neighborhood, or the number of runs required to observe every $k$-marking in a probabilistic setting.

The concept of marking also removes a minor technicality around node removals: should we allow a node $u$ to compute an embedding in a run where $u$ is removed, or should it compute a specific \textsc{invalid} value? In case of markings, $u$ simply remains a node of the network, and it is aware of the fact that it has been marked.

Note that the marking idea is also loosely related to the concept of identity-aware GNNs \cite{Iaware}, where the authors assume that each node can recognize itself (but only itself) in the tree representation of its neighborhood. However, these approaches require a different kind of way to process the tree representation of every node, and furthermore, in contrast to markings, the GNN is often still unable to the recognize that specific nodes in the tree representation actually correspond to the same original node.

\subsection{Proof of Theorem \ref{th:markdrop}}

We now show Theorem \ref{th:markdrop}, i.e. that markings are indeed more expressive than node removals. Intuitively, markings offer the following advantages compared to the node removal approach:
\begin{itemize}[topsep=0pt,itemsep=0pt,partopsep=7pt,parsep=7pt]
\item Whenever multiple nodes are removed, the information about the edges between these nodes is lost. On the other hand, in case of markings, marked nodes are still aware of (and can easily e.g. count) their marked neighbors. Due to this, it is significantly easier to e.g. count cliques with markings than with dropouts.
\item Whenever a node is removed, the GNN is not able to pass information through this node anymore. 
It can happen that two subgraphs are only distinguishable through the dropout of a specific node $v$, however, all paths from a node $u$ in our graph go through this node $v$. As such, $u$ cannot separate the two cases, since whenever $v$ is dropped, $u$ has no access to the rest of this subgraph.
The same problem does not appear in case of markings, where the underlying graph remains intact.
\end{itemize}
We construct a concrete example graph based on the first point.

\renewcommand*{\proofname}{Proof of Theorem \ref{th:markdrop}}
\begin{proof}
We define two graphs $G_1$ and $G_2$ as follows. Consider two independent cycles of length $4$ and $10$, respectively, and number their nodes from $0$ to $3$ and from $0$ to $9$ in a clockwise order. Let us draw an edge from our node $u$ to all the $14$ nodes in these graphs. We then add two extra edges to both $G_1$ and $G_2$. In $G_1$, we connect nodes $0$ and $2$, and nodes $5$ and $7$ in the cycle of length $10$; that is, we add two chords between nodes at distance $2$ in the distant ends of the cycle. In $G_2$, we connect nodes $0$ and $5$ of the $10$-cycle, and we connect nodes $0$ and $2$ of the $4$-cycle (i.e. we add one of the longest possible chords in both cycles). Finally, there are still $10$ nodes in both graphs that have degree $3$ only; to each of these, we add a separate leaf node that we only connect to this specific node. We call these leaves \textit{upper nodes}, while we call the original $14$ nodes \textit{lower nodes}. Furthermore, let us refer to the $4$ lower nodes that do not have an adjacent upper node as \textit{crossing nodes}. Altogether, both $G_1$ and $G_2$ consists of $u$ (with degree $14$), $14$ lower nodes of degree $4$, and $10$ upper nodes of degree $1$.

We consider this graph from $u$'s perspective with $d=2$; note that the entire graph is within the $2$-hop neighborhood of $u$. We consider a GNN with $k=2$; that is, we show that (i) the two graphs cannot be distinguished by a GNN after removing two nodes, but (ii) the two graphs can be distinguished by a GNN after marking two nodes.

Let us first analyze the case of removing nodes. If $k'$ nodes are removed from the $d$-hop neighborhood of $u$ in a run, then let us call this run a $k'$\textit{-removal} (analogously to our definition of a $k'$-marking). Note that similarly to the case of the $C_{\ell, \ell}$ and $C_{2\ell}$ graphs (analyzed later), a GNN can essentially extract the following information from the graph in $d=2$ rounds in a given run: the (remaining) degree of $u$ in the first round, and the multiset of the (remaining) degrees of the nodes adjacent to $u$ in the second round. If both of these coincide for a specific run, then a GNN with node removals is unable to distinguish the given $k'$-removals.

Let us analyze all the node removal patterns in both graphs. Note that without removing nodes, all neighbors of $u$ have degree $4$, so the two graphs are identical to $1$-WL.

As for $1$-removals, there are $24$ of these in both $G_1$ and $G_2$ (assuming that $u$ itself is not removed). In $10$ of these $1$-removals, we delete a lower node which also has an upper neighbor; these cases are all identical for a GNN, since they imply that $u$'s degree decreases by $1$, and two neighbors of $u$ also have their degree decreased by $1$ and $1$ (we will use $(1,1)$ for a short notation of this effect on the neighbors, and call it the \textit{signature} of this removal). Besides this, there are $4$ distinct $1$-removals (in both graphs) where the deleted lower node is a crossing node; this means that $u$'s degree decreases by $1$, and three distinct neighbors of $u$ also lose a degree (i.e. a signature of $(1,1,1)$). Finally, there are $10$ distinct $1$-removals where an upper node is removed; this has no effect on $u$'s degree, and has signature $(1)$. Since all of these patterns have identical multiplicity in $G_1$ and $G_2$, the two graphs cannot be distinguished from the $1$-removals.

Now let us consider $2$-removals, and split this into three cases, based on whether two upper nodes are removed, or two lower nodes, or an upper and a lower node. The simplest case is when two upper nodes are removed: this has no effect on $u$'s degree, has a signature of $(1,1)$, and can occur in $\binom{10}{2}$ different ways in both graphs.

Now assume that one upper and one lower node is removed. Note that all of these cases reduce the degree of $u$ by $1$, so we only need to consider the degrees of $u$'s neighbors. There are $10$ pairs in both graphs where we remove a lower node and its upper neighbor; this results in a pattern of $(1,1)$. Next let us consider the pairs where $v$ is a lower node, and the upper node we remove is adjacent to one of the neighbors of $v$. There are $8$ cases in both graphs where $v$ is a crossing node, which result in $(2,1,1)$. In the rest of the cases ($12$ of them), the signature is always $(2,1)$. Finally, consider the pairs where the distance between the two nodes is at least $3$. When $v$ is a crossing node and the upper node does not belong to a neighbor of $v$, then the signature is $(1,1,1,1)$; this happens in $32$ ways in both graphs. When $v$ is not a crossing node and the upper node does not belong to a neighbor of $v$, then the signature is $(1,1,1)$, and this can happen in $78$ ways.

Finally, assume that two lower nodes are removed; this always decreases the degree of $u$ by $2$. Note that in this case (and this is the main idea of the proof), if the two removed nodes were crossing nodes connected by an edge, then $u$ remains unaware of this. First consider the pairs that are adjacent to each other along one of the cycles (i.e. not through a chord added later); there are $14$ such pairs in both graphs. In both graphs, $4$ of these pairs have signature $(2,1)$: this happens when one of the nodes is a crossing node with a short chord (over an arc of length $2$), and the other node is the node in the middle of this arc. Another $4$ pairs have signature $(1,1,1)$: one of the nodes is still a crossing node, and its neighbor is not part of a short arc; in particular, this always happens in the $10$-cycle, with nodes $3$, $4$, $8$, $9$ (and their crossing neighbor) in $G_1$, and nodes $1$, $4$, $6$, $9$ (and their crossing neighbor) in $G_2$. The remaining $6$ adjacent pairs have signature $(1,1)$.

Now consider the lower node pairs that are at distance $2$ along one of the cycles. In both graphs, these pairs in the $4$-cycle give a signature of $(2,2)$. In both $10$-cycles, there are $4$ such pairs where exactly one the two nodes is a crossing node, which results in $(2,1,1,1)$. The remaining such pairs all produce a signature of $(2,1,1)$, and there are $6$ such pairs.

Finally, consider the lower node pairs that are at distance larger than $2$ along the $10$-cycle, or in different cycles. There are $4$ such pairs where both nodes are crossing nodes and they are not adjacent, which results in a signature of $(1,1,1,1,1,1)$. There are $4$ cases in both graphs where exactly one of the two nodes is a crossing node, and they also have a common neighbor in the $10$-cycle (such as e.g. nodes $0$ and $3$ in $G_1$, or nodes $0$ and $4$ in $G_2$); this results in $(2,1,1,1)$. There are also $28$ pairs where exactly one of the two nodes is crossing, but they have no common neighbor (apart from $u$); this gives $(1,1,1,1,1)$. Finally, there are $33$ cases in $G_1$ where neither of the nodes is crossing, and these give a signature of $(1,1,1,1)$. In $G_2$, there are $32$ corresponding cases of two non-crossing nodes, with signature $(1,1,1,1)$. However, there is also the pair with the two crossing nodes in the $10$-cycle, which provides the same signature of $(1,1,1,1)$, thus increasing its multiplicity to $33$ in $G_2$, too.

Since each case occurs the same number of times in the two graphs, the graphs cannot be distinguished by a GNN based on the set of $1$-removals and $2$-removals. On the other hand, consider the $2$-markings in the graph; in particular, consider the specific $2$-marking in $G_2$ where we mark the two crossing nodes in the $4$-cycle. This comes with a signature of $(2,2)$, i.e. in $d=2$, $u$ will know that it has two neighbors that are both adjacent to both of the marked nodes. Furthermore (and in contrast to node removals), the marked nodes also detect in the first round that they have a marked neighbor, and communicate this to $u$ in the second round. That is, $u$ can recognize this situation by having (i) two unmarked neighbors that are both adjacent to two marked nodes, and (ii) two marked neighbors that are also adjacent to a marked node. The same situation can never occur in $G_1$, since the $4$-cycle has no chord. As such, if there is a $2$-marking where this occurs, then $u$ can deduce that the graph is $G_2$ and not $G_1$.
\end{proof}

\subsection{Proof of Theorem \ref{th:inj}}

Finally, let us consider the maximal expressive power of GNNs with markings. The reference point for these GNNs is the $1$-WL algorithm with colors initialized according to a specific marking. Note that in Section \ref{sec:mark}, we have only defined this for graphs without input features; if input features are also present, then nodes are initialized to a different color for the marked and unmarked version of each input feature.

Whenever two nodes receive the same color in $1$-WL with colors initialized according to a marking, then a GNN with markings computes the same embedding for these two nodes. For this, we need to observe that if the initialization is based on markings, then the colors assigned to marked nodes and the colors assigned to unmarked nodes will remain disjoint during the entire $1$-WL procedure. One can then show our claim with a simple induction: assume that the claim holds up to round $(t-1)$. Then in round $t$, if the immediate neighborhood of two nodes contains the same multiset of colors, then these colors can unequivocally be sorted into two groups (colors belonging to marked and unmarked nodes), and the two groups will be identical for the two nodes. This implies that $\textsc{aggr}_{\text{marked}}$ and $\textsc{aggr}_{\text{unmarked}}$ receives the same input, so the two nodes compute the same $a_u\!^{(t)}$. Since the two nodes also have the same color, in round $(t-1)$, both of them will apply the same update function (out of $\textsc{upd}_{\text{marked}}$ and $\textsc{upd}_{\text{unmarked}}$), and hence they compute the same embedding $h_u\!^{(t)}$.

It remains to show that an appropriate GNN implementation can indeed reach this expressiveness.

\renewcommand*{\proofname}{Proof of Theorem \ref{th:inj}}

\begin{proof}
For the proof of injectiveness, we consider the same assumptions as in case of GIN \cite{gin}: the space of initial features is countable, and there is a known upper bound $L$ on the degree of nodes. Similarly to in case of \cite{gin}, an induction shows that the space of possible embeddings remains countable after any constant number of rounds.

The main idea of the proof is also similar, but it requires some modifications due to our more general setting. After each round, we know that there exists a mapping $Z$ from the space of current possible embeddings to the natural numbers $\mathbb{N}$ (because the set is countable). For $i \in \{ 0, 1, 2, 3\}$, let us define the function $f_i(x)=L^{-4 \cdot Z(x) + i}$, and then for a multiset $X$, let us define
\[ \textsc{aggr}_{\text{marked}}(X) = \sum_{x \in X} f_0(x) \qquad \text{and} \qquad \textsc{aggr}_{\text{unmarked}}(X) = \sum_{x \in X} f_1(x) \, . \]
Furthermore, let \[ \textsc{upd}_{\text{marked}}(h_u\!^{(t-1)}, a_u\!^{(t)}) = f_2(h_u\!^{(t-1)}) + a_u\!^{(t)} \: \text{  and}\] \[ \textsc{upd}_{\text{unmarked}}(h_u\!^{(t-1)}, a_u\!^{(t)}) = f_3(h_u\!^{(t-1)}) + a_u\!^{(t)} \, .\]
The resulting representation allows us to unambiguously reconstruct both $h_u\!^{(t-1)}$ and the multiset of previous adjacent embeddings. The digits at positions $2$ and $3$ modulo $4$ essentially implement a one-hot encoding for the embedding $h_u\!^{(t-1)}$ of $u$ (with the modulus of position also indicating whether $u$ is marked). The digits at positions $0$ and $1$ modulo $4$ encode an $L$-digit representation of the multiset of adjacent embeddings, for marked an unmarked neighbors, respectively, similarly to GIN. As such, whenever two nodes receive a different color in the next round $1$-WL, then they also compute a different embedding in the next round. An induction shows that this holds over any number of rounds. Note that in a practical implementation, the functions $f_i$ can be replaced by a universal approximation tool such as a multi-layer perceptron (MLP).

Note that such a sophisticated function is in fact only required in the first round, when marked nodes have to be separated from unmarked nodes. In the following rounds, the marking information is already indirectly contained in the embeddings of the nodes, so an injective standard GNN (which ignores markings) is also sufficient for all the remaining rounds to ensure injectivity.

Finally, we can select the run-aggregation function to be an injective multiset function; note that this representation technique from \cite{gin} shows exactly that such a function exists, and can be implemented with a combination of an MLP and summation. As such, if two neighborhoods are separable under markings (i.e. the corresponding multisets of final embeddings are not identical), then a GNN with such a run-aggregation function assigns a different final embedding to them.
\end{proof}

\renewcommand*{\proofname}{Proof.}

\section{Proofs of Theorems \ref{th:CvC} and \ref{th:CvC2}: the $C_{\ell, \ell}$ vs. $C_{2\ell}$ construction} \label{app:CvC}

This section discusses the proofs of Theorems \ref{th:CvC} and \ref{th:CvC2}, through a detailed analysis of the $C_{\ell, \ell}$ vs. $C_{2\ell}$ construction. Throughout the analysis, we will assume $d=2$. Note that the entire graphs are within the $2$-hop neighborhood of $u$, and thus $u$ could easily distinguish the two graphs in case of a classical distributed algorithm in two rounds.

Some basic ingredients of the proofs have already been discussed in Section \ref{sec:compare}: $N_1$ can always distinguish the two graphs since both graphs are within the induced $1$-hop neighborhood of $u$ (and they are non-isomorphic), and $S_{(\ell+1)}$ can also distinguish them, since the subgraph consisting of an $\ell$-cycle and a fully connected node is a structure on $(\ell+1)$ nodes that only appears in $C_{\ell, \ell}$.

In order to complete the proofs, the following further ingredients are needed:

\begin{lemma} \label{lem1}
$S_k$ cannot distinguish the two graphs if $k \leq \ell$.
\end{lemma}

\begin{proof}
The tree representations of the graphs observed by $u$ are identical in the two cases, so we only need to show that $u$ is assigned the same extra features (i.e. observes the same subgraphs) in the two cases. Let us analyze the subgraphs of a specific size $k$ (this is a slight abuse of notation, since the $k$ in $S_k$ denotes the maximal size of these subgraphs).

Any subgraph of size $k$ incident to $u$ consists of $(k-1)$ nodes distributed somehow along the cycle(s). Since $u$ is contained in all subgraphs, each subgraph is essentially a collection of paths such that the sum of the length of the paths is $(k-1)$, and then another node is connected to each node of every path. Let $k'=k-1$, and for simplicity, let us call a graph that consists of disjoint paths on a total of $k'$ nodes a \textit{tassel} graph. Each subgraph of size $k$ is completely characterized by such a tassel after discarding $u$ from it. Hence, in order to show that the multiset of adjacent substructures is identical in the two graphs, it suffices to show that if (i) we consider a single $2 \ell$-cycle and the union of two $\ell$-cycles, (ii) we select $k'$ nodes from both graphs in every possible way, and (iii) we consider the corresponding tassels (induced by the selected nodes), then we end up with the same result (same multiset of tassels) in both cases.

Let the $2 \ell$-cycle and the two independent $\ell$-cycles be denoted by $G_1$ and $G_2$, respectively, for simplicity. Let us consider all possible ways to select $k'$ nodes from $G_1$, and denote it by $\mathcal{P} _1$. Let us consider all possible ways to select $k'$ nodes from $G_2$, and denote it by $\mathcal{P} _2$. We show a bijection between $\mathcal{P} _1$ and $\mathcal{P} _2$ such that the corresponding selection of nodes induce the same tassel. Note that both graphs have $2 \ell$ nodes, so the cardinality of both $\mathcal{P} _1$ and $\mathcal{P} _2$ is $\binom{2 \ell}{k'}$.

Let us number the nodes of $G_1$ from $1$ to $2 \ell$ clockwise, and in $G_2$, number the nodes of the two cycles clockwise from $1$ to $\ell$ and from $\ell+1$ to $2\ell$, respectively. A natural starting point is to consider a bijection of nodes with the same number, and for any $k'$-tuple of nodes in $\mathcal{P} _1$, simply assign to it the $k'$-tuple of nodes in $\mathcal{P} _2$ with the same numbers. The problem with this approach is that when both nodes $\ell$ and $\ell+1$ are selected in $G_1$, then this forms a continuous path, but in $G_2$, these nodes are part of different cycles. On the other hand, whenever both nodes $\ell$ and $1$ are selected in $G_2$, then this is a continuous path within the first cycle, but not in $G_1$, where the other neighbor of node $\ell$ is node $\ell+1$ instead of node $1$. Hence with this trivial approach, the bijected pairs from $\mathcal{P} _1$ and $\mathcal{P} _2$ do not always produce the same tassel.

To overcome this, consider the following approach. Given a $k'$-tuple of nodes $p_1 \in \mathcal{P} _1$, if neither node $\ell$ nor node $2\ell$ is contained in $p_1$, then we follow the trivial approach above, i.e. we select the subset of nodes with the same numbers. Intuitively, since the paths of this tassel are interrupted anyway in the points where the two graphs differ, the resulting collection of paths will be the same. On the other hand, if at least one of nodes $\ell$ and $2\ell$ is within the $k'$-tuple $p_1$, then we find the smallest index $i$ such that the following holds: neither node $i$ nor node $\ell+i$ is contained within the $k'$-tuple. Note that $i$ is well-defined, and since we only have $k'$ nodes in $p_1$ with $k'< \ell$, and one of these nodes is either node $\ell$ or $2\ell$, we will certainly have $i \leq \ell-2$. Then, intuitively speaking, we ``swap'' nodes $(1, ..., i-1)$ with nodes $(\ell+1, ..., \ell+i-1)$: that is, if a node $j \in \{ 1, ..., i-1\}$ is selected in $p_1$, then we include node $j+\ell$ in the corresponding selection $p_2 \in \mathcal{P} _2$, and if a node $j \in \{ \ell+1, ..., \ell+i-1\}$ is selected in $p_1$, then we include node $j-\ell$ in $p_2$.

This strategy is indeed a bijection: for any $p_2 \in \mathcal{P} _2$, we can easily find the element in $p_1 \in \mathcal{P} _1$ that this $p_2$ was assigned to. This is relatively simple, since $i \leq \ell-2$, and therefore the elements $\ell$ and $2\ell$ are never swapped. Hence if neither $\ell$ nor $2\ell$ is contained in $p_2$, then $p_2$ was assigned to the same selection of numbers in $\mathcal{P} _1$. On the other hand, if one of them is contained in $p_2$, then we know that a swapping happened, hence we find the index $i$ described above (this is invariant to our swapping), and exchange the node pairs $j$ and $j+\ell$ (for $j \in \{ 1, ..., i-1\}$) in order to compute $p_1$ from $p_2$.

It remains to show that the corresponding pairs produce the same tassel. This is also straightforward from our construction method: whenever a path does not contain nodes $\ell$ and $2\ell$, it is continuously mapped from $p_1$ to $p_2$, and whenever one of these nodes is contained, the swapping operation ensures continuity. Note that indices $i$ and $(\ell+i)$ are a breakpoint between two paths both before and after swapping, so the path segments between $(i+1)$ and $\ell$ and between $(\ell+i+1)$ and $2\ell$ are entirely unaffected by the swapping.

The bijection between the tassels shows that the multiset of subgraphs of size $k$ (incident to $u$) are identical in the two graphs. Since this applies to any $k \leq \ell$, $u$ receives the same features in the two graphs.
\end{proof}

\begin{lemma} \label{lem2}
$M_k$ cannot distinguish the two graphs if $2k < \ell$.
\end{lemma}

\begin{proof}
Recall from the proof of Theorem \ref{th:markdrop} that in $d=2$ rounds, $u$ can essentially extract the following information from the graph: the number of its marked and unmarked neighbors (in round $1$), and for each of these neighbors, the number of marked/unmarked nodes adjacent to this neighbor.

More specifically, if two nodes are marked at distance $2$ along (one of) the cycle(s), then this can still be recognized by the GNN: $u$ then has a neighbor which detected two distinct marked neighbors in round $1$. On the other hand, if the nodes are at distance $2$ along a cycle, then this is already indistinguishable from the case when there is an arbitrary large distance between the two nodes: in both cases, $u$ will only observe two neighbors that only have a single marked neighbor in round $1$, and $u$ has no way to recognize that these two neighbors are also adjacent to each other.

As such, similarly to Lemma \ref{lem1}, we can essentially separate our cycles into arcs by splitting them at every point where two consecutive nodes are unmarked. Note that in contrast to Lemma \ref{lem1}, an arc is not just defined by its length: it can be any sequence of marked an unmarked nodes, with the restriction that it has no two consecutive unmarked nodes. If two arcs are identical, then the GNN receives the same set of messages from these nodes. The remaining nodes outside of the arcs (i.e. having distance at least $2$ to any marked node) always send the same messages to $u$ regardless of their position, since they are not aware of any marking. As such, if two graphs consist of the same multiset of arcs (possibly distributed along the cycle(s) in a different way), then the GNN with markings is unable to distinguish the two graphs.

From here we follow the same proof idea as in Lemma \ref{lem1}: given all the possible sets of arcs $\mathcal{P}_1$ formed when distributing $k$ markings along a $2\ell$-cycle ($G_1$), and all the possible sets of arcs $\mathcal{P}_2$ formed when distributing $k$ markings along two distinct $\ell$-cycles ($G_2$), we show a bijection between $\mathcal{P}_1$ and $\mathcal{P}_2$ that preserves the set of arcs. We number the nodes of $G_1$ and $G_2$ as before. Our criteria for swapping is similar to before: if all of the nodes $(\ell-1)$, $\ell$, $(2\ell-1)$ and $2\ell$ are unmarked, then we again assign every node to its original counterpart. Otherwise, we compute the smallest index $i \geq 0$ such that all of the nodes $(\ell+i)$, $(\ell+i+1)$, $i$ and $(i+1)$ are unmarked, and hence these two pairs of nodes form a valid breakpoint for the arcs (in case of $i=0$, node number $0$ is understood as an alias for node $2 \ell$). Note that any arc can have length at most $(2\cdot k-1)$, so even if we have the longest possible arc starting at node $\ell$ or $2\ell$, we still have $i \leq 2 \cdot k-1 \leq \ell - 2$. We again swap the node pairs $j$ and $\ell+j$ for each $j \in \{ 1, ..., i-1\}$ to find the pair $p_2 \in \mathcal{P}_2$ of a marking $p_1 \in \mathcal{P}_1$.

This is once again a valid bijection: since $i \leq \ell - 2$, the markings at positions $(\ell-1)$, $\ell$, $(2\ell-1)$ and $2\ell$ are never swapped, and hence for any $p_2 \in \mathcal{P}_2$, we can easily reconstruct the $p_1$ it was mapped to. Furthermore, the swapping ensures that each arc is mapped continuously; hence the corresponding pairs of markings result in the same final embedding for $u$.

By applying this proof for all $k' \in \{ 1, ..., k\}$, it follows that $M_k$ obtains the same multiset of embeddings in the two graphs over the set of all runs, so regardless of the run-aggregation function, the final embeddings are identical.
\end{proof}

\begin{lemma} \label{lem3}
$2$-WL can distinguish the two graphs.
\end{lemma}

\begin{proof}
If we consider the two graphs without the fully connected node $u$, then one can easily show that $2$-WL can distinguish the two graphs (e.g. since it is known that $2$-WL can distinguish graphs of treewidth $2$).

Furthermore, whenever two graphs (without a fully connected node) can be distinguished by $2$-WL, then the same holds after a fully connected node is added to both graphs. Intuitively, whenever a pair of original nodes receives a different color in a refinement step in the original graph, they will also receive a different color in the new graphs, since their relationship to each other remains unchanged. As for node pairs that contain the newly added node: these can easily be distinguished already in the first color refinement step (due to the full connectivity of the new node), so they cannot be confused with original node pairs.

Alternatively, one can use the result of Kiefer \cite{kiefer2019power}, which shows that separating pairs of nodes can already be recognized by $2$-WL. Such a pair exists in $C_{\ell, \ell}$, but not in $C_{2\ell}$.
\end{proof}

\section{Further proofs for Sections \ref{sec:comp1}$-$\ref{sec:comp3}} \label{app:other}

We now discuss the proofs for Theorems \ref{th:NWL}--\ref{th:SWL}. We begin with a proof of Theorem \ref{th:NWL}, i.e. that $N_1 \succ k$-WL for any choice of $k$.

\renewcommand*{\proofname}{Proof of Theorem \ref{th:NWL}}

\begin{proof}
One can show this by slightly extending the graph construction of Cai \textit{et al.} \cite{CFI} (we discuss this construction later in the proof of Theorem \ref{th:CFI}); this defines a pair of graphs $G$ and $G'$ that are indistinguishable by $k$-WL.

Let us now add a new node $u$ to these graphs, and connect $u$ to every other node. One can show that these graphs still remain indistinguishable to $k$-WL, following the same line of thought as in the proof of Lemma \ref{lem3}.

However, now the entire original graphs are in the induced $1$-hop neighborhood of $u$. Since $G$ and $G'$ are non-isomorphic, $u$ receives a different new feature in the two graphs, and hence $N_1$ can separate them.
\end{proof}

Our next two theorems can be shown with a relatively simple construction of cycle graphs.

\renewcommand*{\proofname}{Proof of Theorem \ref{th:cycfirst}}

\begin{proof}
Consider two cycles of length $\ell_1 = 2k+2$ and $\ell_2 = 2k+3$, respectively. Recall that cycles of different length are one of the most popular example for graphs that are not distinguishable by standard GNNs. Furthermore, in both graphs, the induced $k$-hop neighborhood of any node is simply a path of length $(2k+1)$, so every node (in both graphs) receives the same extra features with $N_k$.

On the other hand, $2$-WL is already able to distinguish cycles of different length, as mentioned before in Lemma \ref{lem3}.
\end{proof}

\renewcommand*{\proofname}{Proof of Theorem \ref{th:markcycle}}

\begin{proof}
Once again, let us consider two cycles of length $\ell_1 = 2k+2$ and $\ell_2 = 2k+3$, respectively, with $d=k+1$. Recall from Theorem \ref{th:cycfirst} that the extra features of $N_k$ are of no use in this case. Similarly, for $S_k$, the multiset of incident subgraphs is identical in the two graphs (paths up to length $k-1$).

One the other hand, the two graphs can already be distinguished when we have a marked node at distance $(k+1)$ from $u$. In the cycle of length $\ell_1$, this will mean that $u$ observes a marked node at distance $(k+1)$ in both directions in the tree representation of the graph; on the other hand, in the $\ell_2$-cycle, any single marked node will only appear within distance $k+1$ from $u$ in one of the two directions. The message passing phase can easily distinguish these two cases in a sufficiently strong GNN (e.g. an injective one as in Theorem \ref{th:inj}).
\end{proof}

The proof of Theorem \ref{th:rook} has already been outlined in Section \ref{sec:compare}. Since the Rook's $4x4$ and Shrikhande graphs are strongly regular with the same parameters, it is known that they cannot be distinguished by $2$-WL. On the other hand, two markings are enough to distinguish the graphs if $d=2$. Consider a $2$-marking where the marked nodes are adjacent both to $u$ and to each other; there are exactly $6$ such $2$-markings in both graphs. Furthermore, a GNN with $d=2$ can easily verify that this is the case: the marked nodes recognize in the first round that they have a marked neighbor, and then they pass this information on to $u$ in the second round.

Now consider the remaining (unmarked) neighbors of $u$ in both graphs. In the Rook's graph, $u$ has an unmarked neighbor that is itself adjacent to two marked nodes. This situation can also be recognized with $d=2$: the unmarked node concludes in the first round that it has two marked neighbors, and then notifies $u$ in the second round. Based on these cases, $u$ can deduce that its induced $1$-hop neighborhood is a $C_{3,3}$ instead of a $C_6$, and hence it is in the Rook's graph. In contrast to this, in the Shrikhande graph, $u$ will have two unmarked neighbors that are both adjacent only to a single marked node, so they will never send a similar message to $u$.

Recall that Theorem \ref{th:GSN} has already been shown in previous work. It follows in a relatively straightforward way from the graphs in Figure \ref{fig:rook}: the nodes in the Rook's $4x4$ graph are incident to $4$-cliques, while the nodes in the Shrikhande graph are not.

On the other hand, the number of incident triangles and paths of length $2$ is straightforward to deduce in the color refinement step of $2$-WL when it inspects the relationship of an adjacent node pair to all other nodes in the graph. This shows that $2$-WL can compute the extra features available to $S_3$, and hence $S_3$ $\subseteq$ $2$-WL.

Finally, the proof of Theorem \ref{th:SWL} has again been outlined in Section \ref{sec:compare}.

\renewcommand*{\proofname}{Proof details for Theorem \ref{th:SWL}}

\begin{proof}
In case of the containment result, the counts of each substructure apart from the path of length $(k+1)$ (i.e.\ the path on $(k+2)$ nodes) can be directly computed from the isomorphism class of the induced $k$-hop neighborhood. We can also count all the paths on $(k+2)$ nodes that are entirely contained in the induced $k$-hop neighborhood. The only remaining paths are those that have the first $k$ edges within the induced $k$-hop neighborhood, and the last edge outside of it; that is, the penultimate and last nodes of the path are at distances $k$ and $(k+1)$ from $u$, respectively.

In these cases, we can consider the degree of the node at distance $k$ (which is available to a GNN after $d \geq k+1$ rounds), subtract from this the degree of this nodes within the induced $k$-hop neighborhood, and we get the number of edges this penultimate node has to other nodes that are at distance $(k+1)$ from $u$. All such edges will provide a separate path of length $(k+1)$ that is incident to $u$. Note that any such path is indeed induced, i.e. no two nodes in it are adjacent, since otherwise the final node would be reachable from $u$ in less than $(k+1)$ hops.

As a technicality, note that even though the isomorphism class of the induced $k$-hop neighborhood is known, it might be non-trivial to figure out which node in the tree representation around $u$ corresponds to which node in the induced $k$-hop neighborhood. As such, finding the degree of nodes at distance $k$ is not necessarily trivial. To do this, one solution is to compare all the walks of length $k$ from $u$ in both the induced $k$-hop neighborhood graph and the $d$-hop tree representation. That is, if a walk of length $k$ ends in a node that is not at distance $k$ (but closer), then we can infer the degree of this node already from the graph known by $N_k$. As such, we can (i) collect all the walks of length $k$ from $u$ in the preprocessed graph, noting the degree of the final node, (ii) collect all walks of length $k$ in the tree representation, also noting the degree of the final node, and then (iii) subtract the first set from the other to get the degree of all nodes at distance $k$. Then from this we can subtract the edges that go from nodes at distance $k$ to other nodes within the induced $k$-hop neighborhood (i.e.\ to nodes at distance $(k-1)$ or $k$); this gives us the number of ways we can complete our distance-$k$ paths to with edges to distance-$(k+1)$ nodes, and hence the number of paths of length $(k+1)$.

Now consider the second statement in the theorem, i.e. $S_{(k+3)} \succ N_k$. In our example for this, the induced $k$-hop neighborhood of $u$ in both graphs is simply a path of length $(k-1)$. Furthermore, in two more rounds after reaching the end of the path, a standard GNN cannot distinguish the different structures at the end (just like a standard GNN in $2$ rounds cannot distinguish a triangle and two outgoing paths of length $2$ from $u$). As such, the graphs are not distinguished by $N_k$.

On the other hand, $S_{(k+3)}$ can detect the entire graph in $G_1$, whereas $G_2$ contains different structures of size $(k+3)$ (and in particular, none of those contain a triangle). Hence $S_{(k+3)}$ assigns different features to $u$ in the two cases, which is already enough to distinguish the graphs.
\end{proof}

\section{Proof of Theorem \ref{th:CFI}} \label{app:CFI}

Note that using some of our previous constructions, we can prove the claims of Theorem \ref{th:CFI} in a relatively straightforward manner. The cycle graphs of Theorem \ref{th:markcycle} already show an example where $S_k$ requires a parameter choice of $k \geq n_i - O(1)$ (with an extra leaf node added at the farthest point from $u$ in the $\ell_1$-cycle if we insist on having $|G|=|G'|$). The construction for the second part of Theorem \ref{th:SWL} shows an example where $N_k$ needs to have $k \geq n_i - O(1)$ to separate the two graphs. Finally, the $C_{\ell, \ell}$ vs. $C_{2\ell}$ construction of Theorem \ref{th:CvC} requires us to have at least $\frac{n_i-1}{4}=\Omega(n_i)$ nodes marked in order to distinguish the two cases.

We point out that one can also combine these properties in a single graph, with the slight drawback that the difference of the parameters of $S_k$ and $N_k$ will also turn from an additive to a multiplicative constant; that is, the new claim will only state that $G_i$ and $G_i'$ cannot be distinguished with $N_k$ and $S_k$ unless $k \geq n_i \, / \, O(1)$. In the rest of the section, we outline the main idea of a construction that fulfills these properties.

To combine the properties into a single construction, we turn to the CFI graphs devised in \cite{CFI}. For a detailed description of this construction, we refer the reader to the original work of the authors. Intuitively, the construction is based on a graph transformation which replaces each node and edge of an original graph $G_0$ by a specific gadget to obtain a graph $G$, and then ``twists'' one of the edge gadgets to also obtain a twisted graph $G'$. The node and edge gadgets are designed such that the twist can be ``moved around'' in the graph. That is, if the twist is on an incident edge to an original node $v$ of $G_0$, and we untwist this edge and twist another edge that is incident to $v$ instead, then the resulting graph is still isomorphic to $G'$.

This already hints that the graphs $G$ and $G'$ are very hard to distinguish for any isomorphism test: essentially, if our algorithm ignores any edge $e$ of the original graph $G_0$, then by moving the twist to $e$ in $G'$, one can show that $G$ and $G'$ will seem identical to the algorithm.

This transformation already allows us to prove the theorem with the appropriate choice of $G_0$. For our proofs (and to satisfy the assumptions of the transformation), we will require the following properties from $G_0$: it has to be $3$-regular, and it needs to have a radius of $n_0\, / \,O(1)$ (where $n_0$ is the number of nodes of $G_0$). One can easily construct such a graph e.g. for any $n_0$ divisible by $4$: we take a cycle of length $n_0$, and for each $i \in {0, ..., \frac{n}{4}-1}$, we add the extra edges $(4i,4i+2)$ and $(4i+1,4i+3)$ (where nodes are numbered around the cycle). Note that the CFI transformation of this graph $G_0$ maintains the property that the radius of the graph is $n\, / \,O(1)$.

Let us execute the CFI transformation on this graph $G_0$, and select $d$ such that the $d$-hop neighborhood of (any) node $u$ contains the entire graph. Assum without loss of generality that $u$ is chosen within the node gadget corresponding to node $0$ in $G_0$.

Our knowledge of the radius already makes the claim on $N_k$ straightforward: it implies that there exists a constant $c \in O(1)$ such that the induced neighborhood of radius $k=n_i\, / \,c$ around a node $u$ does not contain every edge of the original graph (i.e. every edge gadget after transformation). This implies that we can move the twist to this missing edge of the graph, i.e. relabeling the nodes in the induced $(n_i\, / \,c)$-neighborhood of $u$ shows that this neighborhood is isomorphic in $G$ and $G'$. This shows that every node receives the same extra features, so $N_k$ is only as expressive as $1$-WL on this graph. On the other hand, $1$-WL clearly cannot distinguish $G$ and $G'$ since they are $3$-regular.

This also settles the question for $S_k$ indirectly: since the graphs are isomorphic within this radius, if we select $k=n_i\, / \,c$, the multiset of incident subgraphs (and hence all the newly added features) are identical. Again, the message passing phase is of no help since the graphs are $3$-regular.

For the case of $M_k$, one can show that the markings are indistinguishable unless we mark at least one node in linearly many node gadgets. That is, let us select $k=n_0\,/\,8-O(1)$; this ensures that $k$ is indeed in $\Omega(n_i)$. Let us consider the $k'$-markings of $G$ and $G'$ where $k' \in \{ 0, ..., k \}$. If we show a bijection from these markings in $G$ to the markings in $G'$ such that paired markings produce the same embedding for $u$ in the message passing phase (i.e. they receive the same color under $1$-WL when initialized according to these markings), then the two multisets of embeddings from the different runs is identical, and thus $u$ will compute the same final embedding for any run-aggregation function.

Let $n_0$ be divisible by $8$. Whenever $i$ is divisible by $4$, let us call the segment of the main cycle in $G_0$ from node $i$ to node $(i+3)$ a \textit{block}. Note that if $i$ is the beginning of a block, this means that the extra edge added to node $(i-1)$ (to achieve $3$-regularity) comes from an earlier node, while the extra edge added to $i$ goes to a later node; in other words, deleting the edge $(i-1, i)$ disconnects this part of the cycle.

Now consider edge edge $(\frac{n_0}{2}-1, \frac{n_0}{2})$ of the same main cycle; note that $\frac{n_0}{2}$ is divisible by $4$, so node $\frac{n_0}{2}$ is the beginning of a block. Consider an interpretation of $G'$ (i.e.\ a mapping between the nodes of $G$ and $G'$) where the twisted edge gadget corresponds to this edge $(\frac{n_0}{2}-1, \frac{n_0}{2})$ of $G_0$. Let us consider a $k'$-marking of $G$, and let us define the corresponding $k'$-marking in $G'$ as follows. In $G_0$, consider the position of the twisted edge and that of $u$ (i.e.\ the node in $G_0$ that corresponds to the node gadget containing $u$); these split the $n_0$-cycle of $G_0$ to two arcs of approximately equal size. In one of the arcs, we leave the marking unchanged: a node is marked in $G'$ exactly if its corresponding pair is marked in $G$. In the other arc, starting from the twisted edge, let us consider the first block $G_0$ such that no node is marked in the entire block (any gadgets of it); such a block must exist, since $k' \leq n_0\,/\,8-O(1)$. Let $v$ be the first node of this block in $G_0$ from the direction of the twisted edge. On the arc between $u$ and $v$, we keep the marking unchanged. However, between the twisted edge and $v$ (not including $v$), we follow the edge gadgets along the main cycle in $G_0$, and we swap the role of the $a$-nodes and $b$-nodes with regard to the marking: we mark an $a$-node in $G'$ if the corresponding $b$-node was marked in $G$, and vice versa. See the construction of Cai \textit{et al.} \cite{CFI} for more details on the roles of the specific nodes within the node gadgets.

One can show that the corresponding nodes receive the same color in $1$-WL if initialized with colors according to these markings in $G$ and $G'$ (and hence $u$ computes the same embedding in a GNN with markings). The marking defined in $G'$ essentially amounts to propagating the twisted edge to the first point where an entire node gadget is unmarked in the main cycle. As such, the larger arc of the cycle from $u$ to $v$ (which includes the original twisted edge) behaves identically in the two graphs: their isomorphism from $G$ to $G'$ also preserves the marking we defined. The only parts of $G'$ we have to discuss are (i) the extra edges added to the main cycle (for $3$-regularity) in the arc where the markings were modified, and (ii) the block containing node $v$.

For the extra edges within this arc, one can observe that they are in an identical situation for $1$-WL as in $G$: the automorphisms of the node gadgets in the construction are designed exactly such that the $a$-nodes and $b$-nodes can be swapped on the other two incident edges simultaneously without any effect. As for the block with node $v$, this is also not affected by the fact that the markings are modified up to node $v$; the block does not let this marking information pass through it by design. That is, let us initialize a whole block in $G$ with identical colors, and set the $4$ $a$-nodes and $b$-nodes at the boundaries of the block (connecting it to the rest of the graph) to arbitrary colors. One can verify that if we run $1$-WL on (i) this graph, and (ii) on the same graph after exchanging the colors of the $a$-node and $b$-node at one end of the block, then the nodes within the block receive the same final color in both cases. As such, the different marking pattern up to node $v$ has no effect within the block or besides $v$ (i.e.\ on the shorter arc between $v$ and $u$).

This shows that the corresponding nodes will indeed receive the same color in $1$-WL, and hence $u$ computes the same embedding in $G$ and $G'$.

\section{Proofs for Section \ref{sec:count}} \label{app:counting}

We now discuss our proofs on counting cliques and induced cycles with our GNN extensions. As a simple definition of counting, we can say that an extension counts a specific subgraph if there exists a GNN implementation where the following holds: whenever two nodes have a different number of incident cliques/cycles (up to some reasonable upper bound $L$), their final embedding is also different. Note that by applying a sufficiently powerful update function in the last round, we can also convert such an implementation to a GNN that actually assigns the number of incident subgraphs to $u$ as its final embedding.

We also note that we focus on \textit{induced} cycles because they have a more prominent role in some applications; however, our observations also carry over to counting cycles in general.

\subsection{Counting with $S_k$}

The first half of Theorem \ref{th:count_S} is straightforward: $S_k$ is directly provides $u$ with the number of incident $k$-cliques and induced $k$-cycles incident to $u$ as extra features.

To show that $S_k$ cannot count $(k+1)$-cycles, we can simply consider the proof of Theorem \ref{th:markcycle} with cycles of length $\ell_1=k+1$ and $\ell_2=k+2$, respectively. Since $S_k$ is only aware of paths of up to $k$ nodes, and the tree representations are identical, it cannot distinguish the two cases, even though the number of incident $(k+1)$-cycles is different.

Note that the result on counting cliques is also tight (i.e. $S_k$ cannot count $(k+1)$-cliques) for small $k$ values, such as $k=0$ and $k=1$. In particular, the graphs in Figure \ref{fig:WLcounter} and the $C_{3,3}$ vs. $C_6$ graphs show that $S_2$ and $S_3$ cannot count $3$-cliques and $4$-cliques, respectively.

\subsection{Counting with $N_k$}

From Theorem \ref{th:count_N}, the first statement is again straightforward: any $\ell$-clique (for any $\ell \geq 3$) is entirely contained within the induced $1$-hop neighborhood of $u$. Hence there is a well-defined function $f$ which assigns the appropriate number of $\ell$-cliques to any extra feature of $u$ (i.e. any induced $1$-hop neighborhood), and a sufficiently powerful GNN (e.g. with an injective \textsc{update} function) can compute $f$.

Similarly, the claim on counting $(2k+1)$-cycles follows from the fact that every induced $(2k+1)$-cycle is entirely contained in the induced $k$-hop neighborhood of $u$.

Finally, the claim on counting $(2k+2)$-cycles follows from the proof of Theorem \ref{th:markcycle} again: if we consider two cycles of length $\ell_1=2k+2$ and $\ell_2=2k+3$, respectively, then $N_k$ will compute the same final embedding for $u$ in the two cases.

\subsection{Counting with $M_k$}

In $M_k$, consider a run where $k$ distinct nodes (not including $u$) of the $(k+2)$-clique are marked. In the first round, each of the marked nodes can indeed confirm that it has $(k-1)$ marked neighbors. In the second round, consider the node $v$ in the clique which is unmarked and also not identical to $u$: this node can decide if it received $k$ distinct messages from $k$ marked nodes which all claim to have $(k-1)$ marked neighbors each. Finally, in the third round, if $u$ is informed of this situation by its neighbor $v$, and also has $k$ more adjacent marked nodes with $(k-1)$ marked neighbors each, then it can conclude that it is contained in a $(k+2)$-clique.

Note that it might also happen in the third round that $u$ has $k$ adjacent marked nodes, and it receives such a message from multiple nodes $v_1$, ..., $v_{\ell}$; this implies that there are $\ell$ different adjacent $(k+2)$-cliques that contain this $k$-marking.

With this approach, an injective GNN with markings can count the number of incident cliques from the set of all possible $k$-markings. Note that with this method, each incident clique is counted $(k+1)$ times, so we have to divide the final count by $(k+1)$ for the correct result. This finishes the proof of Theorem \ref{th:count_M}.

Once again, the graphs in Figure \ref{fig:WLcounter} and the $C_{3,3}$ vs. $C_6$ graphs show that the result is tight for $k=0$ and $k=1$, i.e. $M_0$ and $M_1$ cannot count $3$-cliques and $4$-cliques, respectively.

We note, however, that in this case, it is not straightforward to also transfer this result to cliques of smaller size. $M_k$ can still easily count cliques of size $(k+1)$ and $k$ by marking all nodes (apart from $u$ for $(k+1)$). However, for cliques of size $\ell<k$, we can only use this approach if the GNN is able to infer the number of marked nodes in a run, i.e. if it explicitly knows or recognizes that there are currently only $(\ell-2)$ marked nodes.

On the other hand, counting induced cycles with markings is a more involved question. What we can still prove here is the following.
\begin{lemma}
An $M_k$ GNN can count $(k+1)$-cycles if $d \geq k+1$.
\end{lemma}

This lemma is easy to show: consider the case when each node of the induced cycle (except for $u$) is marked. In this case, the marked neighbors of $u$ can deduce in the first round that they only have a single marked neighbor. In the next round, their immediate marked neighbors can deduce that they have an outgoing marked path of $2$ nodes; if they have exactly $2$ marked neighbors, then they can communicate this to their neighbors and continue this process. With each node checking the number of its marked neighbors, the neighbors of $u$ find out after $k$ rounds that they are on the ends of a marked path of length $k$. If $u$ has $2$ such neighbors, then it can conclude that it is in an induced $(k+1)$ cycle in round $(k+1)$.

Note, however, that this is in some sense a significantly weaker result than what we had in most of the analyses in the paper. That is, even if the entire induced cycle is contained in e.g.\ the $2$-hop neighborhood of $u$, the GNN still has to pass a message around the cycle to recognize it with this method. It would be much more appealing to already be able to recognize the cycle from these markings as soon as its nodes are all contained within the $d$-hop neighborhood of $u$. However, this is not straightforward, since even if all nodes of the cycle (apart from $u$) are marked, it still remains challenging to decide if they form a single large cycle, or several smaller ones.

\section{Adjustments for graph classification} \label{app:global}

Finally, let us discuss the generalizations of our theorems to a graph classification setting. Note that whenever two $d$-hop neighborhoods can be distinguished by a node $u$ in an extension, then they are also distinguishable on a graph level from the different embedding of $u$. Hence we only need to discuss the cases when an extension cannot distinguish two neighborhoods around $u$, to ensure that the graphs cannot be distinguished in this case by the embeddings of the remaining nodes either.

Note that in many of our constructions (e.g. the Rook's $4x4$ / Shrikhande graphs), the role of each node is symmetric, so indistinguishability from a specific node's perspective also carries over to the whole graph, and the constructions require no modification. A same holds for the cycle graphs, apart from the fact that here the two graphs have different size; to make them indistinguishable on a graph level, we can take $\ell_2$ distinct copies of the $\ell_1$-cycle, and $\ell_1$ distinct copies of the $\ell_2$-cycle as our two new graphs (both on $\ell_1 \cdot \ell_2$ nodes). This settles the case of Theorems \ref{th:cycfirst}--\ref{th:GSN}.

Theorem \ref{th:NWL}  also requires no change, since $k$-WL is already known to be unable to distinguish these graphs even in a global setting. The same holds for the containment part (first claim) of Theorem \ref{th:SWL}, which also carries over without difficulty.

As such, we only need to revisit the proofs that are based on the $C_{\ell, \ell}$ vs. $C_{2\ell}$ construction (Theorems \ref{th:CvC} and \ref{th:CvC2}), the construction showing $S_{(k+3)}$ $\succ$ $N_k$ (the second part of Theorem \ref{th:SWL}), and the counting results in Section \ref{sec:count}.

\subsection{$C_{\ell, \ell}$ vs. $C_{2\ell}$ construction}

First consider the case of $S_k$, i.e. the claims $N_1$ $\succ$ $S_k$ and $2$-WL $\succ$ $S_k$ from Theorem \ref{th:CvC}. Note since we now consider the $d$-hop neighborhood of all nodes, we first of all have to change our example to $d=3$ to fulfill our assumption that the preprocessing phase does not go beyond the nodes that are reachable in the message passing phase (note that the entire graph is within the $3$-hop neighborhood of any node).

Furthermore, we have to change the condition in Lemma \ref{lem1} to $k<\ell$ in this case; for $k=\ell$, the cycle nodes in $C_{\ell, \ell}$ would be aware of the presence of the $\ell$-cycle, and could thus distinguish the graphs. However, once we select a value $\ell>k$, the proof works as before, since no node can preprocess an entire $\ell$-cycle. More specifically, each node along the cycles will have the same subgraph counts in the two graphs. For subgraphs not containing $u$, this is easy to see, since these are always subgraphs chosen from within a cycle of length $\ell>k$. For subgraphs containing $u$, one can show this analogously to Lemma \ref{lem1}: the proof also applies if we restrict ourselves to selections in $\mathcal{P}_1$ and $\mathcal{P}_2$ where a specific node is always selected. This is easiest to see if we place this special node to a non-swapping position in the cycles, e.g.\ as node number $\ell$.

Now consider the case of $M_k$, i.e. the claims $N_1$ $\succ$ $M_k$ and $2$-WL $\succ$ $M_k$ in Theorem \ref{th:CvC}. Here we can leave $d=2$ as before; however, we have to make sure that the cycle nodes do not reach the other end of the cycle in $d$ rounds, since this would allow them to identify the $\ell$-cycle with a single marking (as in Theorem \ref{th:markcycle}). That is, we must always select $\ell>2 \cdot d = 4$, so we only consider this construction with $\ell \geq 5$.

However, with this restriction, one can show that Lemma \ref{lem2} carries over to this case. In particular, we know that the fully connected node $u$ computes the same embedding in every round for the two graphs, so the remaining nodes receive no useful information from $u$ to distinguish the two graphs. On the other hand, with node $u$ disregarded, the $d$-hop neighborhood of each cycle node is a path of length $2d$, which exhibits the same possible marking configurations for any of the remaining nodes in either of the two graphs. Hence each cycle node will also compute the same final embedding in case of $M_k$.

It only remains to discuss Theorem \ref{th:CvC2}, which compares $S_k$ to $M_j$. In this case, our assumptions on $S_k$ forces us to select $d=3$; hence to make the graphs indistinguishable to $M_j$, we need to choose $\ell>2 \cdot d = 6$. If this $\ell \geq 7$ holds, then the graphs are indeed indistinguishable by $M_j$, as discussed before. In fact, since the nodes in the $\ell$-cycle are already aware of the $\ell$-cycle with $S_{\ell}$ (and recall that we choose $\ell=2k+1$), our construction even proves the slightly tighter result of $S_{(2k+1)}$ $\succ$ $M_k$ in the graph classification setting. The only special cases are $k=1$ and $k=2$, where we cannot choose $\ell=2k+1$ due to $\ell \geq 7$. As such, all that follows from this proof in regard to these cases is that $S_7$ $\succ$ $M_1$ and $S_7$ $\succ$ $M_2$.

We point out that this line of thought also shows the same slightly stronger result for node classification for $k \geq 3$: that is, looking at the same construction from the perspective of one of the cycle nodes, it follows that $S_{(2k+1)}$ $\succ$ $M_k$ for any $k \geq 3$. We have decided to still present Theorem \ref{th:CvC2} in Section \ref{sec:compare} in its current, slightly weaker form since it covers the cases $k=1$ and $k=2$, which are significantly more relevant in practice, and are also the cases that are visible in Figure \ref{fig:sum}.

\subsection{Showing $S_{(k+3)}$ $\succ$ $N_k$}

This claim is the only one that is significantly different for graph classification. The graph in Theorem \ref{th:SWL} was specifically designed to have the indistinguishable part as far from $u$ as possible, and it seems to be non-trivial to generalize such a scenario to a graph classification setting, i.e.\ to construct a graph for general $k$ where every node is in a similar situation.

As such, the only straightforward graph where one can show that $S_i$ is superior to $M_j$ in a graph classification sense is a cycle graph on $(2k+2)$ nodes: this cannot be distinguished by $N_k$, but it can easily be separated by $S_{(2k+2)}$. Hence in this case, we can only show a weaker result between these two extensions, namely that $S_{(2k+2)}$ $\succ$ $N_k$ for any $k \geq 1$.

\subsection{Results on counting}

Note that the positive results on counting carry over to the graph classification setting easily: if each node knows the number of incident copies of a subgraph $G$, then we only need to sum up these numbers over all nodes and divide it by the size of the subgraph. Hence it only remains to discuss the negative results in our theorems. For showing the we cannot count $(k+1)$-cycles with $S_k$, we have used cycles of length $\ell_1$ and $\ell_2$; we can again adjust these by taking many independent cycles (on altogether $\ell_1 \cdot \ell_2$ nodes). The same holds for the proof that we cannot count $(2k+2)$-cycles with $N_k$.

\end{document}